\documentclass[12pt]{article}
\usepackage{fullpage}

\usepackage{amsmath,amsfonts,amssymb}
\usepackage{algorithm,algorithmic}
\usepackage{multirow}
\usepackage{cellspace}
\usepackage{setspace}
\usepackage{bm}
\usepackage{bbm}
\usepackage{subfigure}
\usepackage{graphicx}
\usepackage{tikz,pgfplots}
\usepackage{framed}
\usepackage{amsthm}

\usepackage[round]{natbib}
\setcitestyle{authoryear,round,citesep={;},aysep={,},yysep={;}}

\usepackage{hyperref}
\hypersetup{colorlinks,
            linkcolor=blue,
            citecolor=blue,
            urlcolor=magenta,
            linktocpage,
            plainpages=false}

\usepackage[capitalise]{cleveref}
\usepackage{times}
\usepackage{graphicx} % more modern
\usepackage{subfigure} 

\usepackage{natbib}

\usepackage{algorithm}
\usepackage{algorithmic}

\usepackage{hyperref}
\usepackage{geometry}                % See geometry.pdf to learn the layout options. There are lots.
\usepackage[parfill]{parskip}    % Activate to begin paragraphs with an empty line rather than an indent
\usepackage{subfigure}
\usepackage{graphicx}
\usepackage{amssymb}
\usepackage{epstopdf}
\usepackage{algorithm}
\usepackage{algorithmic}
\newtheorem{theorem}{Theorem}[section]

\newtheorem{lemma}{Lemma}[section]

\newcommand{\A}{\mathcal{A}}

\newcommand{\G}{\mathcal{G}}
\newcommand{\tf}{\tilde{f}}

\newcommand{\F}{\mathcal{F}}
\newcommand{\tg}{\tilde{g}}

\newcommand{\tO}{\tilde{O}}

\newcommand{\blr}[1]{\big(#1\big)}

\newcommand{\eq}{~=~}
\renewcommand{\le}{~\leq~}
\renewcommand{\ge}{~\geq~}

\DeclareMathOperator*{\argmin}{arg\,min}

\def\reals{{\mathcal R}}

\newcommand{\K}{\mathcal{K}}
\newcommand{\R}{\mathcal{R}}

\newcommand{\ignore}[1]{}

\def\reals{{\mathbb R}}

\def\bold0{\mathbf{0}}

\newcommand\E{\mbox{\bf E}}

\include{header}

\title{Online to Offline Conversions, Universality and Adaptive Minibatch Sizes}%
\author{%
Kfir Y. Levy\footnote{Department of Computer Science, ETH Z\"urich. 
Email :\texttt{yehuda.levy@inf.ethz.ch}.}
}

\begin{document}
\maketitle

\begin{abstract} 
We present an approach towards convex optimization that relies on a  novel scheme  which converts   online adaptive algorithms into  offline methods.
 In the offline optimization setting, our derived methods are shown to obtain  favourable adaptive guarantees which depend on the \emph{harmonic sum} of the queried gradients.  We further show that our methods implicitly adapt to the objective's structure:  
in the smooth case   fast convergence rates are ensured without any prior knowledge of the smoothness parameter, while still maintaining guarantees in the non-smooth setting.
%This contrasts with line-search GD (gradient descent) procedures which do not hold guarantees in the non-smooth setting.
%require in to know advance whether the problem is smooth/non-smooth.
% Our derived method implicitly adapts to the objective's structure, attaining adaptive bounds which depend on the harmonic sum of gradient
% 
% 
% attaining  fast convergence rates without any prior knowledge of the smoothness.
%This contrasts with the GD (gradient descent) algorithm which requires the smoothness parameter in order to obtain such fast rates.
%These rates are comparable to the ones attained by the GD (gradient descent) algorithm had it known the smoothness parameter.
Our approach has a natural extension to the stochastic setting, resulting in a lazy version of SGD (stochastic GD), where minibathces are chosen \emph{adaptively}  depending on the magnitude of the gradients. Thus providing a principled approach towards choosing minibatch sizes.
\end{abstract}

\section{Introduction}

%When employing a first order  method, one is required to choose a learning rate rule which controls the scale of the steps taken  in every dimension, and
%it is known that a smart tuning of the learning rate is crucial to the success of the learning process.
%Thus, state-of-the-art implementations employ methods like AdaGrad \cite{duchi2011adaptive}, and Adam \cite{kingma2014adam}, which adapt the learning rate on the fly according to the feedback (i.e. gradients) received during the optimization process. 
%Despite the success of existing adaptive schemes, the assumptions underlining their theoretical basis do not capture well the task of Neural Network optimization. We intend to develop new adaptive schemes based on a more suitable set of assumptions. 

%Over the past years \emph{data adaptive} methods like Adagrad [] Adam[] Adadelta[], were proven to be extremely successful
%in tackling large scale problems.
%The objective function underlying 
% ``big data" applications often demonstrates intricate structure:
% the scale and smoothness are often unknown and may change substantially in between different regions/directions.
% %Furthermore, the scale/smoothness is often unknown.
% Learning methods that  acclimatize to these changes may exhibit a 
% superior performance compared to non adaptive procedures, which in turn might make the difference  between success and failure in practice. 
% 
% Though  Adagrad [] Adam[], are guaranteed to 

Over the past years
\emph{data adaptiveness} has proven to be  crucial to the success of learning algorithms. 
The objective function underlying 
 ``big data" applications often demonstrates intricate structure:
 the scale and smoothness are often unknown and may change substantially in between different regions/directions.
 %Furthermore, the scale/smoothness is often unknown.
 Learning methods that  acclimatize to these changes may exhibit  
 superior performance compared to non adaptive procedures, which in turn might make the difference  between success and failure in practice (see e.g. \cite{duchi2011adaptive}).

 State-of-the-art  first order methods like AdaGrad, \cite{duchi2011adaptive}, and Adam, \cite{kingma2014adam},  adapt the learning rate on the fly according to the feedback (i.e. gradients) received during the optimization process. 
 AdaGrad and Adam are guaranteed to work well in the \emph{online} convex optimization setting, where loss functions may be chosen \emph{adversarially}  and change between rounds.
Nevertheless, this setting is  harder than the stochastic/offline settings, which may better depict practical applications.  Interestingly, even in the offline convex optimization setting it could be shown that in several scenarios
very simple schemes  may substantially outperform the output of  AdaGrad/Adam.
An example of such a simple scheme is choosing the point with the smallest gradient norm among all rounds.
In the first part of this work we address this issue and design adaptive methods for the offline convex optimization setting.
At heart of our derivations is a novel  scheme which converts  online adaptive algorithms into offline methods with favourable guarantees\footnote{For concreteness we concentrate in this work on converting AdaGrad,~\cite{duchi2011adaptive}. Note that our conversion scheme applies more widely  to  other  online adaptive methods.}. 
Our shceme is inspired by  standard online to batch conversions as introduced in the seminal work of \citet*{cesa2004generalization}.

%Despite the success of existing adaptive schemes, the assumptions underlining their theoretical basis do not capture well the task of Neural Network optimization. We intend to develop new adaptive schemes based on a more suitable set of assumptions.
A seemingly different issue is choosing the minibatch size, $b$, in the stochastic setting.
Stochastic optimization algorithms that  can access a noisy gradient oracle may choose to invoke  the oracle $b$ times in every query point, subsequently employing an averaged gradient estimate. 
Theory for stochastic convex optimization   suggests to use a minibatch of $b=1$, and predicts a degradation  of 
$\sqrt{b}$ factor upon using larger minibatch sizes~\footnote{A degradation by a  $\sqrt{b}$ factor in the general case and by a  $b$  factor in the strongly-convex case.}.
%In the general stochastic convex setting theory suggests to use a minibatch of $b=1$, where lar. 
Nevertheless in practice larger minibatch sizes are usually found to be effective. %, despite the theoretical degradation.
In the second part of this work we design stochastic optimization methods in which minibatch sizes are chosen \emph{adaptively}
 without any theoretical degradation. These  are natural extensions of the offline methods presented in the first part. 
 
 Our contributions: 
 \paragraph{Offline setting:} We present two (families of) algorithms $\text{AdaNGD}$ (Alg.~\ref{algorithm:AdaNGD}) and $\text{SC-AdaNGD}$ (Alg.~\ref{algorithm:SC-AdaNGD})   for the convex/strongly-convex settings which achieve favourable adaptive guarantees (Thms.~\ref{thm:AdaNGDnonsmooth},~\ref{thm:AdaNGD2},~\ref{thm:SC-AdaNGDnonsmooth},~\ref{thm:SC-AdaNGD2} ). The latter theorems also establish their  universality, i.e., their ability to implicitly take advantage of the objective's  smoothness and attain  rates as fast as 
 GD would have achieved if the smoothness parameter was known.\\
 Concretely, without the knowledge of the smoothness parameter our algorithm ensures an $O(1/\sqrt{T})$  rate in  general convex case and an $O(1/T)$ rate if the objective is also smooth (Thms.~\ref{thm:AdaNGDnonsmooth},~\ref{thm:AdaNGD2}). In the strongly-convex case our algorithm ensures 
  an $O(1/T)$ rate in general and an  $O(\exp(-\gamma T))$ rate if the objective is also smooth (Thm.~\ref{thm:SC-AdaNGD2} ), where $\gamma$ is the condition number.
% To the best of our knowledge the latter result is  novel.
% \kl{say something about the same proof structure for all methods, something like:
% we get the full spectrum of rates for convex optimization $1/sqrt{T}, 1/T, exp(-\gamma T)$ using the same basic analysis tool, they are all based on the regret guarantees of AdaGrad for the online setting}
 %with the explicit knowledge of the smoothness.\\
%\vspace{5pt}

 \paragraph{Stochastic setting:} We present  Lazy-SGD (Algorithm~\ref{algorithm:SAdaNGD2_new}) which is an extension of our offline algorithms. Lazy-SGD employs larger minibatch sizes in points with smaller gradients, which selectively reduces the variance in the ``more important" query points. 
%Moreover, this algorithm modifies the learning rate according to the minibatch sizes. 
%\kl{think of rephrasing the variance reduction statement-see what happens with experiments}
 Lazy-SGD guarantees are comparable with SGD in the convex/strongly-convex settings (Thms.~\ref{lem:LazySGD_general_expectation},~\ref{lem:LazySGDstronglyConvex_expectation}).
%%In the realizable case of  square loss stochastic  optimization\footnote{Note we do not assume strong-convexity.}   we design a different method,
%%Fast Lazy-SGD (Alg.~\ref{algorithm:FastLazySGD}).   This first order algorithm achieves a convergence rate of
% $\tO(1/T^{2/3})$ within $T$ noisy queries (Thm.~\ref{thm:FLazySGD_general_expectation}), which improves upon the standard $O(1/\sqrt{T})$ rate of SGD. While second order methods such as Online Newton Step may achieve  rates as fast as  $O(d/T)$, their per query runtime  scales like $d^3$ with the dimension~\cite{Koren}.
% Conversely, the per query runtime of Fast Lazy-SGD is equivalent to that  of SGD.

On the technical side, our online to offline conversion schemes employ three simultaneous mechanisms: 
 an  online adaptive algorithm used in conjunction with 
gradient normalization and with a respective importance weighting.
To the best of our knowledge the combination of the above techniques is novel, and we believe it might also find use in other scenarios.
%On the technical side, our adaptive methods employ three simultaneous mechanisms: 
%gradient normalization in conjunction with a respective importance weighting, and learning rate adaptation.
%While the latter is a well known adaptive strategy,  we believe that the combination of the above techniques is novel and might find use in other scenarios.
%%Moreover, in the stochastic setting we encounter an interesting phenomenon: 
%%Lazy-SGD needs to adapt the learning rate depending on the minibatch sizes.
%%Conversely, in Fast Lazy-SGD the learning rate is independent of the minibacthes.

This paper is organized as follows. 
In Sections~\ref{sec:AdaNGD_general},\ref{sec:AdaNGD_StronglyConvex},  we present our methods for the offline convex/strongly-convex settings.  Section~\ref{sec:Stochastic} describes our methods  for the stochastic setting.
In Section~\ref{sec:Extensions} we discuss several extensions, and  Section~\ref{sec:experiments} presents a preliminary experimental study. Section~\ref{sec:Discussion} concludes.
%in contrast to existing first order methods which adaptive mechanism is through  

\subsection{Related Work}
%Method that adapt the learning rate  are by now standard techniques... toolkit....  
\citet*{duchi2011adaptive},  simultaneously to \citet{mcmahan2010adaptive}, were the first to suggest AdaGrad---an adaptive gradient based method, and prove its efficiency in tackling online convex  problems.
AdaGrad was subsequently adjusted to  the deep-learning setting to yield the 
RMSprop, \cite{tieleman2012lecture}, and Adadelta, \cite{zeiler2012adadelta}, heuristics. 
\citet{kingma2014adam}, combined ideas from AdaGrad together with momentum machinery, \cite{nesterov1983method}, and devised Adam---a popular adaptive algorithm which is often the method of choice in deep-learning applications.

%Adam \cite{}, is another  popular adaptive algorithm which combines ideas from AdaGrad together with momentum machinery \cite{}, and is often the method of choice in deep-learning applications. 

An optimization procedure is called universal if it implicitly adapts to the objective's smoothness.
In \cite{nesterov2015universal}, universal gradient methods are devised for the general convex setting. 
Concretely, without the knowledge of the smoothness parameter, these methods  attain the standard $O(1/T)$ and accelerated $O(1/T^2)$ rates for smooth objectives, and an $O(1/\sqrt{T})$ rate in the non-smooth case. The core technique in this work  is a line search procedure which estimates the smoothness parameter in every iteration. 
For strongly-convex and smooth objectives, line search techniques,\cite{nocedal}, ensure linear convergence rate, \emph{without} the knowledge of the smoothness parameter. However, line search is not ``fully universal", in the sense that it holds no guarantees in the non-smooth case.
For the latter setting we present a method which ``fully universal" (Thm.~\ref{thm:SC-AdaNGD2}), nevertheless it \emph{requires} the strong-convexity parameter.
Composite optimization methods, \cite{composite}, may obtain fast rates  even for non-smooth objectives. Nevertheless,  proximal-GD may separately access the  gradients of the "smooth part" of the objective, which is a more refined notion than the normal oracle access to the (sub-)gradient of the whole objective.
%

%Nevertheless, this technique requires exact zero order information, which makes it less appropriate for stochastic settings. 

The usefulness of employing normalized gradients was demonstrated in several non-convex  scenarios.
In the context of  quasi-convex  optimization, \cite{nesterov1984minimization}, and \cite{hazan2015beyond}, established  convergence guarantees for the offline/stochastic settings. More recently, it was shown in \cite{levy2016power},  that normalized gradient descent is more appropriate than GD for saddle-evasion scenarios. 

In the context of stochastic optimization,  the effect of minibatch size was extensively investigated throughout the past years, \cite{dekel2012optimal,cotter2011better,shalev2013accelerated,li2014efficient,takavc2015distributed,jain2016parallelizing}. Yet, all of these studies: $\textbf{(i)}$ assume a smooth expected loss, $\textbf{(ii)}$ discuss fixed minibatch sizes. Conversely, our work discusses adaptive minibatch sizes, and applies to both smooth/non-smooth expected losses.

%The usefulness of gradient normalization  for   quasi-convex objectives was demonstrated in \cite{nest} and in \cite{Us}, where convergence guarantees were established for the offline/stochastic settings. More recently, it was shown in \cite{saddle}  that normalized gradient descent is more appropriate than GD for saddle-evasion scenarios.

%
%It is well known that in the stochastic convex optimization setting usi
% Due to their simplicity and computational efficiency, first order algorithms are the methods of choice  for 
%  challenging Machine Learning scenarios. 
% Naturally, such methods may adjust to the data through adapting the learning rate , i.e., through changing the scale of steps in between regions/directions. This applies to many popular procedures e.g. Adagrad [] adam[] adadelta[],  and universal GD[].
%\kl{say something about our technique - normalization-lr adaptation - improtance weighting}
%  
% Adaptive first order methods like 
%  
% So far the most well known mechanism thg
% First order methods like AdaGrad, Adam and Universal GD adapt to 
   
%Some of the most challenging machine learning scenarios involve 
%Finding methods that adapt to the data was proven to be 

\subsection{Preliminaries}
\paragraph{Notation:}  $\|\cdot\|$ denotes the $\ell_2$ norm, $G$ denotes a bound on the norm of the objective's gradients, and $[T]:= \{1,\ldots,T\}$.
For a  set $\K\in\reals^d$ its diameter is defined as $D = \sup_{x,y\in\K}\|x-y\|$.
Next we define $H$-strongly-convex/$\beta$-smooth functions,
%A convex function $f:\K\mapsto\reals$ is $H$-strongly convex over $\K$ if, 
\begin{align*}
%&\textbf{$H$-strong-convexity,}\\
&f(y) \geq f(x) + \nabla f(x)^\top(y-x) + \frac{H}{2}\|x - y\|^2;\quad \forall x,y \in \K  \quad \textbf{($H$-strong-convexity)} \\
%&\textbf{$\beta$-smoothness,}\\
&f(y) \leq f(x) + \nabla f(x)^\top(y-x) + \frac{\beta}{2}\|x - y\|^2 ;\quad \forall x,y \in \K \quad \textbf{($\beta$-smoothness)} 
\end{align*}

\subsubsection{AdaGrad}
%%%AdaGrad- Alg
%%%%%%%%%%%%%%%%%%%%%%%%%%%%%%%%%%%%%%%%%%%%%%%%%%%%%%%%%%%%%%
\begin{algorithm}[t]
\caption{Adaptive  Gradient Descent ($\text{AdaGrad}$) }
\label{algorithm:AdaGrad}
\begin{algorithmic}
\STATE \textbf{Input}: \#Iterations $T$, $x_1\in \reals^d$, set $\K$ 
\STATE {Set}: $Q_0 =0 $
\FOR{$t=1 \ldots T$ }
\STATE {Calculate:} $g_t= \nabla f_t(x_t)$
\STATE {Update:}
$$Q_{t} = Q_{t-1} + \|g_t\|^2 $$
\STATE {Set} $\eta_t = D/\sqrt{2Q_t}$
\STATE {Update:}
$$x_{t+1}= \Pi_{\K}\left( x_{t}-\eta_t {g}_{t}\right)$$
\ENDFOR
\end{algorithmic}
\end{algorithm}
The adaptive methods presented in this paper lean on AdaGrad (Alg.~\ref{algorithm:AdaGrad}), a well known online optimization method which employs an adaptive learning rate.
The following theorem states $\text{AdaGrad}$'s guarantees,~\cite{duchi2011adaptive},
%%%AdaGrad- guarantee
%%%%%%%%%%%%%%%%%%%%%%%%%%%%%%%%%%%%%%%%%%%%%%%%%%%%%%%%%%%%%%
\begin{theorem}\label{theorem:AdaGrad}
Let $\K$ be a convex set with diameter $D$. Let $\{f_t\}_{t=1}^T$ be an arbitrary sequence 
of convex loss functions. Then Algorithm~\ref{algorithm:AdaGrad} guarantees the following regret;
\begin{align*}
\sum_{t=1}^T f_t(x_t) - \min_{x\in\K}\sum_{t=1}^T f_t(x) \le \sqrt{2D^2\sum_{t=1}^T \|g_t\|^2}~.
\end{align*}
\end{theorem}

%%%General Convex Case
%%%%%%%%%%%%%%%%%%%%%%%%%%%%%%%%%%%%%%%%%%%%%%%%%%%%%%%%%%%%%%
%%%%%%%%%%%%%%%%%%%%%%%%%%%%%%%%%%%%%%%%%%%%%%%%%%%%%%%%%%%%%%
%%%%%%%%%%%%%%%%%%%%%%%%%%%%%%%%%%%%%%%%%%%%%%%%%%%%%%%%%%%%%%
%%%%%%%%%%%%%%%%%%%%%%%%%%%%%%%%%%%%%%%%%%%%%%%%%%%%%%%%%%%%%%%%%%%%%%%%%%%%%%%%%%%%%%%%%%%%%%%%%%%%%%%%%%%%%%%%%%%%%%%%%%%%%%%%%%%%%%%%%%%%%%%%%%%%%%%%%%%%%%%%%%%%%%%%%%%%%%%%%%%%%%%%%%
\section{Adaptive Normalized Gradient Descent (AdaNGD)} \label{sec:AdaNGD_general}
In this section we discuss the convex optimization setting and introduce our  $\text{AdaNGD}_k$ algorithm.
We first derive a general convergence rate which holds for any $k\in \reals$.  Subsequently, we elaborate on  the 
$k=1,2$ cases which exhibit universality as well as adaptive guarantees that may be substantially better compared to standard methods.

Our method $\text{AdaNGD}_k$   is depicted in Alg.~\ref{algorithm:AdaNGD}. 
This algorithm can be thought of as an  \emph{online to offline conversion scheme} which utilizes  AdaGrad (Alg.~\ref{algorithm:AdaGrad}) as a black box and eventually outputs a weighted sum of the online queries. Indeed, for a fixed $k\in\reals$, it is not hard to notice that $\text{AdaNGD}_k$ is equivalent to invoking AdaGrad with the following loss sequence $\{\tf_t(x) := {g_t^\top x}/{\|g_t\|^k} \}_{t=1}^T$.
And eventually weighting each query point inversely proportional to the $k$'th power norm of its gradient.
The reason behind this scheme is that in offline optimization it makes sense to dramatically  
reduce the learning rate upon uncountering a point with a very small gradient. 
For $k\geq1$, this is achieved by invoking AdaGrad with gradients normalized by their $k$'th power norm.
Since we discuss constrained optimization, we use the projection operator
defined as, 
$$\Pi_\K(y):=\min_{x\in \K}\|x-y\| ~.$$
%Our $\text{AdaNGD}_k$  algorithm is presented in Algorithm~\ref{algorithm:AdaNGD}. Note that the update rule employs gradients normalized by their $k$'th power norm, which are also used to set the learning rate.
%Moreover, the algorithm's output is a weighted average of the query points, and each weight is inversely proportional to the $k$'th power norm of the gradient.
%%Note that $\text{AdaNGD}_k$ outputs a weighted average of the query points, and each weight is inversely proportional to the $k$'th power norm of the gradient, also note that algorithm employs gradients normalized by their $k$'th power norm. 
%The parameter $k\in \reals$ is an input for the algorithm.
%Since we discuss constrained optimization, we use the projection operator
%defined as, 
%$$\Pi_\K(y)~:=~\min_{x\in \K}\|x-y\| ~.$$
%%%AdaNGD_k Alg
%%%%%%%%%%%%%%%%%%%%%%%%%%%%%%%%%%%%%%%%%%%%%%%%%%%%%%%%%%%%%%
\begin{algorithm}[t]
\caption{Adaptive Normalized Gradient Descent ($\text{AdaNGD}_{k}$) }
\label{algorithm:AdaNGD}
\begin{algorithmic}
\STATE \textbf{Input}: \#Iterations $T$, $x_1\in \reals^d$, set $\K$ , parameter $k$
\STATE {Set}: $Q_0 =0 $
\FOR{$t=1 \ldots T-1$ }
\STATE {Calculate:} $g_t= \nabla f(x_t),\; \hat{g}_t = {g_t}/{\|g_t\|^k}$
\STATE {Update:}
$$Q_{t} = Q_{t-1} + {1}/{\| g_t\|^{2(k-1)}} $$
\STATE {Set} $\eta_t = D/\sqrt{2Q_t}$
\STATE {Update:}
$$x_{t+1}= \Pi_{\K}\left( x_{t}-\eta_t \hat{g}_{t}\right)$$
\ENDFOR
\STATE \textbf{Return}:  
$ \bar{x}_T= \sum_{t=1}^T\frac{1/\|g_t\|^{k}}{\sum_{\tau=1}^T 1/\|g_\tau\|^{k}}x_t$
\end{algorithmic}
\end{algorithm}
%%%AdaNGD_k lemma
%%%%%%%%%%%%%%%%%%%%%%%%%%%%
%%%%%%%%%%%%%%%%%%%%%%%%%%%%
The following lemma  states the guarantee of AdaNGD for a general $k$:
\begin{lemma}\label{lem:AdaNGDnonsmooth_k}
Let  $k\in  \reals$, $\K$ be a convex set with diameter $D$, and $f$ be a convex function; 
Also let  $\bar{x}_T$ be the output of  $\text{AdaNGD}_k$ (Algorithm~\ref{algorithm:AdaNGD}), then the following holds:
\begin{align*}
f( \bar{x}_T)- \min_{x\in\K}f(x) 
\le 
\frac{\sqrt{2 D^2\sum_{t=1}^T 1/\|g_t\|^{2(k-1)}   }}{\sum_{t=1}^T 1/\|g_t\|^k} %\leq \frac{GD}{\sqrt{T}}  
\end{align*}
\end{lemma}
\begin{proof}[Proof sketch]
Notice that the $\text{AdaNGD}_k$ algorithm is equivalent to applying AdaGrad to the following loss sequence: $\{\tf_t(x) := {g_t^\top x}/{\|g_t\|^k} \}_{t=1}^T$. Thus, applying Theorem~\ref{theorem:AdaGrad}, and using the definition of $\bar{x}_T$ together with Jensen's inequality the lemma follows.
\end{proof}

For $k=0$,  Algorithm~\ref{algorithm:AdaNGD} becomes  AdaGrad (Alg.~\ref{algorithm:AdaGrad}).
Next we focus on the cases where $k=1,2$, showing improved adaptive rates and universality compared to GD/AdaGrad.
These improved rates are  attained thanks to the adaptivity of the learning rate: 
when query points with small gradients are encountered, $\text{AdaNGD}_k$ (with $k\geq 1$)  reduces the learning rate, thus focusing on the region around these points. The hindsight weighting further emphasizes points with smaller gradients.

\subsection{$\text{AdaNGD}_1$}
Here we show that $\text{AdaNGD}_1$ enjoys a rate of $O(1/\sqrt{T})$  in the non-smooth convex setting, and a fast rate of $O(1/T)$ in the smooth setting. We emphasize that the same algorithm enjoys these rates simultaneously, without any prior knowledge of the smoothness or of the gradient norms.

From Algorithm~\ref{algorithm:AdaNGD} it can be noted that for $k=1$ the learning rate becomes independent of the gradients, i.e. $\eta_t = D/\sqrt{2t}$,
the update is made according to the direction of the gradients, and the weighting is inversely proportional to the norm of the gradients. 
 The following  Theorem establishes the guarantees of $\text{AdaNGD}_1$ (see  proof in Appendix~\ref{app:adangd1}),
 %\newpage
 %%%AdaNGD_1 lemma
%%%%%%%%%%%%%%%%%%%%%%%%%%%%
%%%%%%%%%%%%%%%%%%%%%%%%%%%%
 \begin{theorem}\label{thm:AdaNGDnonsmooth}
Let  $k=1$, $\K$ be a convex set with diameter $D$, and $f$ be a convex function; 
Also let  $\bar{x}_T$ be the outputs of  $\text{AdaNGD}_1$ (Alg.~\ref{algorithm:AdaNGD}), then the following holds:
\begin{align*}
f( \bar{x}_T)- \min_{x\in\K}f(x) 
\le
 \frac{\sqrt{2 D^2T  }}{\sum_{t=1}^T 1/\|g_t\|}  \le  \frac{\sqrt{2}GD}{\sqrt{T}} ~.
\end{align*}
Moreover, if $f$ is also $\beta$-smooth and the global minimum $x^* = \arg\min_{x\in\reals^n}f(x)$ belongs to $\K$, then:
\begin{align*}
f( \bar{x}_T)- \min_{x\in\K}f(x) 
\le
\frac{D\sqrt{T}}{\sum_{t=1}^T 1/\|g_t\|} 
\le
  \frac{4\beta D^2}{T}~. 
\end{align*}
\end{theorem}
\begin{proof}[Proof sketch]
The data dependent bound is a direct corollary of Lemma~\ref{lem:AdaNGDnonsmooth_k}.
The general case bound holds directly by using $\|g_t\|\leq G$.
The bound for the smooth case is proven by showing $\sum_{t=1}^T \|g_t\| \leq O(\sqrt{T})$.
This translates to a lower bound $\sum_{t=1}^T 1/\|g_t\|\geq \Omega(T^{3/2})$, which concludes the proof.
\end{proof}

The data dependent bound in Theorem \ref{thm:AdaNGDnonsmooth} may be substantially better compared to the bound of the GD/AdaGrad.  As an example, assume that
 half of the gradients encountered during the run of the algorithm are of   $O(1)$ norms,  and the other gradient norms decay
 proportionally to $O(1/t)$. In this case the guarantee of GD/AdaGrad is $O(1/\sqrt{T})$, whereas 
 $\text{AdaNGD}_1$ guarantees a bound that behaves like $O(1/T^{3/2})$.
Note that the above example presumes that all algorithms encounter the same gradient magnitudes, which might be untrue. Nevertheless in the smooth case  $\text{AdaNGD}_1$ provably benefits due to its adaptivity. 

\textbf{GD Vs  $\text{AdaNGD}_1$ in the smooth case:}
In order to achieve the $O(1/T)$ rate for smooth objectives GD employs a \emph{constant} learning rate, $\eta_t=1/\beta$. 
It is worth to compare this GD algorithm with $\text{AdaNGD}_1$  in the smooth case:
%It is worth to compare $\text{AdaNGD}_1$ 
For both methods the steps become smaller as the algorithm progresses.
However, the mechanism is different: in GD the learning rate is constant,  but the gradient norms decay;
 while in $\text{AdaNGD}_1$ the learning rate is decaying, but the norm of the normalized gradients is constant.

\subsection{$\text{AdaNGD}_2$}
Here we show that $\text{AdaNGD}_2$ enjoys comparable guarantees to $\text{AdaNGD}_1$
in the general/smooth case.
Similarly to $\text{AdaNGD}_1$ the same algorithm enjoys these rates simultaneously, without any prior knowledge of the smoothness or of the gradient norms.
 The following  Theorem establishes the guarantees of $\text{AdaNGD}_2$  (see  proof in Appendix~\ref{app:adangd2}),
 %%%AdaNGD_2 lemma
%%%%%%%%%%%%%%%%%%%%%%%%%%%%
%%%%%%%%%%%%%%%%%%%%%%%%%%%%
\begin{theorem}\label{thm:AdaNGD2}
Let  $k=2$, $\K$ be a convex set with diameter $D$, and $f$ be a convex function; 
Also let  $\bar{x}_T$ be the outputs of  $\text{AdaNGD}_2$ (Alg.~\ref{algorithm:AdaNGD}), then the following holds:
\begin{align*}
f( \bar{x}_T)- \min_{x\in\K}f(x) \leq \frac{\sqrt{2 D^2  }}{\sqrt{\sum_{t=1}^T 1/\|g_t\|^2}} \le 
\frac{\sqrt{2}GD}{\sqrt{T}} ~.
\end{align*}
Moreover, if $f$ is also $\beta$-smooth and the global minimum $x^* = \arg\min_{x\in\reals^n}f(x)$ belongs to $\K$, then:
\begin{align*}
f( \bar{x}_T)- \min_{x\in\K}f(x) 
\le
\frac{\sqrt{2 D^2  }}{\sqrt{\sum_{t=1}^T 1/\|g_t\|^2}} 
\le 
\frac{4\beta D^2}{T}
~.
\end{align*}
\end{theorem}
\begin{proof}[Proof sketch]
The data dependent bound is a direct corollary of Lemma~\ref{lem:AdaNGDnonsmooth_k}.
The general case bound holds directly by using $\|g_t\|\leq G$.
The bound for the smooth case is proven by showing $\sum_{t=1}^T 1/\|g_t\|^2 \geq \Omega(T^2)$, which concludes the proof.
\end{proof}

It is interesting to note that $\text{AdaNGD}_2$ will have always performed better than AdaGrad, had  both algorithms encountered the same gradientÊ norms. This is due to the well known inequality between arithmetic and harmonic means, \cite{bullen2013means},
%@book{bullen2013means,
%  title={Means and their Inequalities},
%  author={Bullen, Peter S and Mitrinovic, Dragoslav S and Vasic, M},
%  volume={31},
%  year={2013},
%  publisher={Springer Science \& Business Media}
%}
% which states that for any sequenceof non-negative numbers $\{a_1,\ldots,a_T \}$ the following applies:
\begin{align*}
\frac{1}{T}\sum_{t=1}^Ta_t \geq \frac{1}{\frac{1}{T}\sum_{t=1}^T 1/a_t}, \qquad \forall \{a_t\}_{t=1}^T \subset \reals_{+}~,
\end{align*}
which directly implies,
\begin{align*} %\label{eq:Ineq_AdaNGD2AdaGrad}
\frac{1  }{\sqrt{\sum_{t=1}^T 1/\|g_t\|^2}}
 \le
 \frac{1}{T} \sqrt{\sum_{t=1}^T \|g_t\|^2}~.
\end{align*}

%which means that $\text{AdaNGD}_2$ performs better than AdaGrad, given that both encounter the same gradients. 
%The guarantee of  $\text{AdaNGD}_2$ might be substantially better than the guarantee of $\text{AdaNGD}_1$:
%consider the case where all gradients encountered during the run of the algorithm are of $O(1)$ norms,
%but for a single gradient with an $O(1/T)$ norm. Then  $\text{AdaNGD}_2$ guarantees  a bound of $O(1/T)$ on the excess loss,
%yet  $\text{AdaNGD}_1$/AdaGrad/GD  only guarantee an $O(1/\sqrt{T})$ bound.

%\newpage
%%%Strongly-Convex
%%%%%%%%%%%%%%%%%%%%%%%%%%%%%%%%%%%%%%%%%%%%%%%%%%%%%%%%%%%%%%
%%%%%%%%%%%%%%%%%%%%%%%%%%%%%%%%%%%%%%%%%%%%%%%%%%%%%%%%%%%%%%
%%%%%%%%%%%%%%%%%%%%%%%%%%%%%%%%%%%%%%%%%%%%%%%%%%%%%%%%%%%%%%
%%%%%%%%%%%%%%%%%%%%%%%%%%%%%%%%%%%%%%%%%%%%%%%%%%%%%%%%%%%%%%%%%%%%%%%%%%%%%%%%%%%%%%%%%%%%%%%%%%%%%%%%%%%%%%%%%%%%%%%%%%%%%%%%%%%%%%%%%%%%%%%%%%%%%%%%%%%%%%%%%%%%%%%%%%%%%%%%%%%%%%%%%%
\section{Adaptive NGD for Strongly Convex Functions} \label{sec:AdaNGD_StronglyConvex}
\label{sec:AdaNGDStronglyCvx}
Here we discuss the offline optimization setting of strongly convex objectives. We introduce our  $\text{SC-AdaNGD}_k$ algorithm,
and present  convergence rates for general $k\in \reals$.  Subsequently, we elaborate on  the 
$k=1,2$ cases which exhibit universality as well as adaptive guarantees that may be substantially better compared to standard methods.

 %%%SC-AdaNGD_k ALG
%%%%%%%%%%%%%%%%%%%%%%%%%%%%
%%%%%%%%%%%%%%%%%%%%%%%%%%%%
\begin{algorithm}[t]
\caption{Strongly-Convex  AdaNGD ($\text{SC-AdaNGD}_{k}$) }
\label{algorithm:SC-AdaNGD}
\begin{algorithmic}
\STATE \textbf{Input}: \#Iterations $T$, $x_1\in \reals^d$, set $\K$, strong-convexity $H$, parameter $k$
\STATE {Set}: $Q_0 =0 $
\FOR{$t=1 \ldots T-1$ }
\STATE {Calculate:} $g_t= \nabla f(x_t),\; \hat{g}_t = {g_t}/{\|g_t\|^k}$
\STATE {Update:}
$$Q_{t} = Q_{t-1} + {1}/{\| g_t\|^{k}} $$
\STATE {Set} $\eta_t = {1}/{HQ_t}$
\STATE {Update:}
$$x_{t+1}= \Pi_{\K}\left( x_{t}-\eta_t \hat{g}_{t}\right)$$
\ENDFOR
\STATE \textbf{Return}:  
$ \bar{x}_T= \sum_{t=1}^T\frac{1/\|g_t\|^{k}}{\sum_{\tau=1}^T 1/\|g_\tau\|^{k}}x_t$
\end{algorithmic}
\end{algorithm}

Our $\text{SC-AdaNGD}_k$ algorithm is depicted in Algorithm~\ref{algorithm:SC-AdaNGD}. Similarly to its non strongly-convex counterpart, $\text{SC-AdaNGD}_k$ can be thought of as an online to offline conversion scheme which utilizes an online algorithm which we denote $\text{SC-AdaGrad}$
%Our $\text{SC-AdaNGD}_k$ algorithm is depicted in Algorithm~\ref{algorithm:SC-AdaNGD}. Similarly to its non strongly-convex counterpart, the output is a weighted sum of the query points,  emphasizing points with smaller gradients. 
%However $\text{SC-AdaNGD}_k$ chooses the learning rate differently from  $\text{AdaNGD}_k$, in a way that exploits the strong-convexity of the objective.
The next Lemma states the guarantee of  $\text{AdaNGD}_k$,
%%%SC-AdaNGD_k lem
%%%%%%%%%%%%%%%%%%%%%%%%%%%%
%%%%%%%%%%%%%%%%%%%%%%%%%%%%
\begin{lemma}\label{lem:SC-AdaNGDnonsmooth2}
Let  $k\in  \reals$, and $\K$ be a convex set. Let $f$ be an $H$-strongly-convex function; 
Also let  $\bar{x}_T$ be the outputs of  $\text{SC-AdaNGD}_k$ (Alg.~\ref{algorithm:SC-AdaNGD}), then the following holds:
\begin{align*}
f( \bar{x}_T)- \min_{x\in\K}f(x) 
\leq
 \frac{1}{2H\sum_{t=1}^T \|g_t\|^{-k}}\sum_{t=1}^T \frac{\| g_t\|^{-2(k-1)}}{\sum_{\tau=1}^t \|g_\tau \|^{-k}}~.
\end{align*}
\end{lemma}
\begin{proof}[Proof sketch]
In Appendix~\ref{app:scadangd_General} we present and analyze $\text{SC-AdaGrad}$. This is an \emph{online} first order  algorithm for strongly-convex functions in which the learning rate decays according to $\eta_t = 1/\sum_{\tau=1}^t H_\tau$, where $H_\tau$ is the strong-convexity parameter of the
 loss function at time $\tau$.
 Then we show that 
 $\text{SC-AdaNGD}_k$  is equivalent to applying $\text{SC-AdaGrad}$ to the following loss sequence: 
 $$\left\{\tf_t(x) = \frac{1}{\|g_t\|^k}g_t^\top x+\frac{H}{2\|g_t \|^k}\|x-x_t\|^2 \right\}_{t=1}^T~.$$
Applying the regret guarantees of $\text{SC-AdaGrad}$ to the above sequence implies the following to hold for any $x\in\K$:
\begin{align} \label{eq:RegretStronglycVX}
\sum_{t=1}^T \tf_t(x_t)-\sum_{t=1}^T\tf_t(x) 
\le 
\frac{1}{2H}\sum_{t=1}^T \frac{\| g_t\|^{-2(k-1)}}{\sum_{\tau=1}^t \|g_\tau \|^{-k}}~.
\end{align}

The lemma then follows by the
   the definition of $\bar{x}_T$ together with Jensen's inequality.
\end{proof}
For $k=0$, $\text{SC-AdaNGD}$ becomes the standard GD algorithm which uses learning  rate of $\eta_t={1}/{Ht}$.
Next we focus on the cases where $k=1,2$.
\subsection{$\text{SC-AdaNGD}_1$}
Here we show that $\text{SC-AdaNGD}_1$ enjoys an $\tO(1/T)$ rate for strongly-convex convex objectives, and a faster rate of $\tO(1/T^2)$ assuming that the objective is also  smooth. 
We emphasize that the same algorithm enjoys these rates simultaneously, without any prior knowledge of the smoothness or of the gradient norms.
%From Algorithm~\ref{algorithm:AdaNGD} it can be noted that for $k=1$ the learning rate becomes independent of the gradients, i.e. $\eta_t = D/\sqrt{2t}$,
%the update is made according to the direction of the gradients, and the weighting is inversely proportional to the norm of the gradients. 
 The following theorem 
 establishes the guarantees of $\text{SC-AdaNGD}_1$  (see  proof in Appendix~\ref{app:scadangd1}),
,
 %%%SC-AdaNGD_1 lem
%%%%%%%%%%%%%%%%%%%%%%%%%%%%
%%%%%%%%%%%%%%%%%%%%%%%%%%%%
 \begin{theorem}\label{thm:SC-AdaNGDnonsmooth}
Let  $k=1$, and $\K$ be a convex set. Let $f$ be a $G$-Lipschitz and $H$-strongly-convex function; 
Also let  $\bar{x}_T$ be the outputs of  $\text{SC-AdaNGD}_1$ (Alg.~\ref{algorithm:SC-AdaNGD}), then the following holds:
\begin{align*}
f( \bar{x}_T)- \min_{x\in\K}f(x) 
&\le 
\frac{G\left(1+\log\left( \sum_{t=1}^T \frac{G}{\|g_t \|} \right)\right)}{2H \sum_{t=1}^T \frac{1}{\|g_t \|}} \\
&\le
\frac{G^2(1+\log T)}{2H T}
~.
\end{align*}
Moreover, if $f$ is also $\beta$-smooth and the global minimum $x^* = \arg\min_{x\in\reals^n}f(x)$ belongs to $\K$, then,
\begin{align*}
f( \bar{x}_T)- \min_{x\in\K}f(x) 
%&\le
%\frac{G\left(1+\log\left( \sum_{t=1}^T \frac{G}{\|g_t \|} \right)\right)}{2H \sum_{t=1}^T \frac{1}{\|g_t \|}}\\
& \le
  \frac{(\beta/H)G^2\left(1+ \log T\right)^2}{H T^2}~.
\end{align*}
\end{theorem}

\subsection{$\text{SC-AdaNGD}_2$}
Here we show that $\text{SC-AdaNGD}_2$ enjoys an $\tO(1/T)$ rate for strongly-convex convex objectives, and a faster rate of $\tO(\exp(-\gamma T))$ assuming that the objective is also smooth. 
We emphasize that the same algorithm enjoys these rates simultaneously, without any prior knowledge of the smoothness or of the gradient norms.
In the case where $k=2$ the guarantee of SC-AdaNGD is as follows  (see  proof in Appendix~\ref{app:scadangd2}),
,
%%%SC-AdaNGD_2 lem
%%%%%%%%%%%%%%%%%%%%%%%%%%%%
%%%%%%%%%%%%%%%%%%%%%%%%%%%%
\begin{theorem}\label{thm:SC-AdaNGD2}
Let  $k=2$, $\K$ be a convex set,  and $f$ be a $G$-Lipschitz and $H$-strongly-convex function; 
Also let  $\bar{x}_T$ be the outputs of  $\text{SC-AdaNGD}_2$ (Alg.~\ref{algorithm:SC-AdaNGD}), then the following holds:
\begin{align*}
f( \bar{x}_T)- \min_{x\in\K}f(x) 
&\le
 \frac{1+\log(G^2 \sum_{t=1}^T\|g_t\|^{-2})}{2H\sum_{t=1}^T \|g_t\|^{-2}} \\
 &\le
\frac{G^2(1+\log T)}{2H T}~.
\end{align*}
Moreover, if $f$ is also $\beta$-smooth and the global minimum $x^* = \arg\min_{x\in\reals^n}f(x)$ belongs to $\K$, then,
\begin{align*}
f( \bar{x}_T)- \min_{x\in\K}f(x) 
%&\le
% \frac{1+\log(G^2 \sum_{t=1}^T\|g_t\|^{-2})}{2H\sum_{t=1}^T \|g_t\|^{-2}} \\
 &\le
\frac{3G^2}{2H }e^{-\frac{H}{\beta}T}\left(1+\frac{H}{\beta}T\right)~.
\end{align*}
\end{theorem}

\textbf{Intuition:} Recall that for strongly-convex objectives the appropriate GD algorithm utilizes two very extreme learning rates for the general/smooth settings. While for the first setting the learning rate is decaying, $\eta_t \propto 1/t$,  for the smooth case  it is constant, $\eta_t=1/\beta$. A possible explanation to the universality of  $\text{SCAdaNGD}_2$ is that it implicitly interpolate between these rates. Indeed the update rule of our algorithm can be written as follows,
$$x_{t+1} = x_t - \frac{1}{H}\frac{\|g_t\|^{-2}}{\sum_{\tau=1}^t \|g_\tau\|^{-2}} g_t~. $$
Thus, ignoring the hindsight weighting, $\text{SCAdaNGD}_2$ is equivalent to GD  with an adaptive learning rate 
$\tilde{\eta}_t: = {\|g_t\|^{-2}}/{H\sum_{\tau=1}^t \|g_\tau\|^{-2}} $.
Now, when all gradient norms are of the same magnitude, then $\tilde{\eta}_t \propto 1/t$, which boils down to the standard GD for strongly-convex objectives. On the other hand, assume that the gradients are exponentially decaying, i.e., that $\|g_t\|\propto q^t$  
for some $q<1$. In this case $\tilde{\eta}_t$ is approximately constant. We believe that the latter applies for strongly-convex and smooth objectives. Nevertheless, note that this intuition is still unsatisfactory since it applies to any $k>0$ and does not explain why $k=2$ is unique (and maybe this is not the case).

%%%Stochastic
%%%%%%%%%%%%%%%%%%%%%%%%%%%%%%%%%%%%%%%%%%%%%%%%%%%%%%%%%%%%%%
%%%%%%%%%%%%%%%%%%%%%%%%%%%%%%%%%%%%%%%%%%%%%%%%%%%%%%%%%%%%%%
%%%%%%%%%%%%%%%%%%%%%%%%%%%%%%%%%%%%%%%%%%%%%%%%%%%%%%%%%%%%%%
%%%%%%%%%%%%%%%%%%%%%%%%%%%%%%%%%%%%%%%%%%%%%%%%%%%%%%%%%%%%%%%%%%%%%%%%%%%%%%%%%%%%%%%%%%%%%%%%%%%%%%%%%%%%%%%%%%%%%%%%%%%%%%%%%%%%%%%%%%%%%%%%%%%%%%%%%%%%%%%%%%%%%%%%%%%%%%%%%%%%%%%%%%
%\newpage 
%\kl{Say something about laziness.., and relevance for the distributed setting... how about strongly-convex and smooth losses? can we say something there?}\\
%\kl{Distinguish between $T$ and $S$, emphasize this!!!}
%%Moreover, in the realizable case of stochastic optimization with square loss functions, we devise an adaptation of 
%%$\text{AdaNGD}_1$ which achieves a runtime of $dcvdsv$. To  the best of our knowledge is the is bla bla bla.
%\vspace{-5pt}
%and it is known that these rates hold both in expectation and with high-probability.

\section{Adaptive NGD for Stochastic Optimization} \label{sec:Stochastic}
Here we show that using data-dependent minibatch sizes, we can adapt our (SC-)$\text{AdaNGD}_2$ algorithms (Algs.~\ref{algorithm:AdaNGD},~\ref{algorithm:SC-AdaNGD} with $k=2$) to the stochastic setting, and achieve the well know convergence rates for the convex/strongly-convex settings.  Next we introduce the stochastic optimization setting,  and then we present and discuss our Lazy SGD algorithm.

\paragraph{Setup:}
We consider the problem of minimizing a convex/strongly-convex function $f:\K \mapsto\reals$, where $\K\in \reals^d$ is a convex set.
We assume that optimization lasts for $T$ rounds; on each round $t=1,\ldots,T$, we may query a point $x_t\in\K$, and receive a \emph{feedback}.
After the last round,  we choose $\bar{x}_T\in\K$, and our performance measure is  the expected excess loss, defined as,
$$\E[f(\bar{x}_T)]- \min_{x\in\K}f(x)~.$$
Here we assume that our feedback is a first order noisy oracle $\G:\K \mapsto \reals^d$ such that upon   querying $\G$ with a point $x_t\in\K$, we receive a bounded 
 and unbiased gradient estimate, $\G(x_t)$, such  $\E[\G(x_t)\vert x_t] = \nabla f(x_t)$; $\|\G(x_t)\|\leq G$. 
We also assume that the  that  the internal coin tosses (randomizations) of the oracle are independent. It is well known that variants of Stochastic Gradient Descent (SGD) are ensured to output an estimate $\bar{x}_T$ such that the excess loss is bounded by $O(1/\sqrt{T})/O(1/T)$ for the setups of convex/strongly-convex stochastic optimization,~\cite{nemirovskii1983problem}, \cite{hazan2007logarithmic}.\\

\textbf{Notation:}
 In this section we make a clear distinction between the number of queries to the gradient oracle, denoted henceforth by $T$; and between the number of iterations in the algorithm, denoted henceforth by $S$. We care about the dependence of the excess loss in $T$.

%%%%%%LazySGD%%%%%%%%%%%%%%%%%%
%%%%%%%%%%%%%%%%%%%%%%%%%%%%%
%%%%%%%%%%%%%%%%%%%%%%%%%%%%%
\begin{algorithm}[t]
\caption{Lazy Stochastic Gradient Descent (LazySGD) }
\label{algorithm:SAdaNGD2_new}
\begin{algorithmic}
\STATE \textbf{Input}: \#Oracle Queries $T$, $x_1\in \reals^d$,  set $\K$,   $\eta_0$, $p$
\STATE {Set}: $t=0,\;  s=0$
\WHILE{$t\leq  T$ }
\STATE {Update:} $s=s+1$
\STATE Set $\G = \text{GradOracle}(x_s)$, i.e., $\G$ generates i.i.d. noisy samples of $\nabla f(x_s)$
\STATE {Get}:   $(\tg_s,n_s) = \text{AE}(\G, T-t)$ %$(\tg_s,n_s) = \text{AE}(\G, T-t)$ 
\tiny{           \text{             \% Adaptive Minibatch }}
\normalsize{ \text{}}
\STATE {Update:} $t=t+n_s$
\STATE {Calculate:} $\hat{g}_s = n_s \tg_s$
\STATE {Set}: $\eta_s=\eta_0/t^p $
\STATE {Update:}
$x_{s+1}= \Pi_{\K}\left( x_{s}-\eta_s \hat{g}_{s}\right)$
\ENDWHILE
\STATE \textbf{Return}:  
$ \bar{x}_T= \sum_{i=1}^s\frac{n_i}{T}x_i~.$ 
\tiny{           \text{             (Note that $\sum_{i=1}^s n_i = T$)}}
\end{algorithmic}
\end{algorithm}
%%%AE
%%%%%%%%%%%%%%%%%%%%%%%%%%%%%%%%%%%%%%%%%%%
%%%%%%%%%%%%%%%%%%%%%%%%%%%%%%%%%%%%%%%%%%%
%%%%%%%%%%%%%%%%%%%%%%%%%%%%%%%%%%%%%%%%%%%
\begin{algorithm}[t]
\caption{ Adaptive Estimate (AE)}
\label{algorithm:AE_new}
\begin{algorithmic}
\STATE \textbf{Input}: random vectors generator $\G$, sample budget $T_{\max}$, sample factor $m_0$
\STATE {Set}: $i =0, N=0$, $\tg_0=0$
\WHILE{   $N < T_{\max}$ } 
\STATE Take  $\tau_i = \min\{2^i,T_{\max}-N\}$ samples from $\G$%; let $\tg$ be their average
\STATE Set $N \gets N+\tau_i$ 
\STATE Update: 
$$ \tg_N \gets \text{Average of N samples received so far from } \G$$ 
%$\tg_N \gets \frac{n-1}{2n-1} \tg^{(n-1)} + \frac{n}{2n-1} \tg^{(n)} $
%\STATE Update estimate: 
%$$\tg^{(n)} \gets \frac{n-1}{2n-1} \tg^{(n-1)} + \frac{n}{2n-1} \tg^{(n)} $$
\STATE \textbf{If} $\|\tg_N\| > 3  m_0/\sqrt{N}$ \textbf{then return} $(\tg_N, N)$
\STATE Update $i\gets i+1$ 
\ENDWHILE
%\STATE  Take  $(\tau - 2^{i}+1)$ samples from $\G$; let $\tg^{(\text{final})}$ be their average
%\STATE  Update $\tg = \frac{\tau - 2^{i}+1}{\tau - 2^{i}+2^{i-1}+1}\tg^{(\text{final})}   + \frac{2^{i-1}}{\tau - 2^{i}+2^{i-1}+1}\tg^{(i)}$
%%%\STATE Take $n$ samples from $\G$, let $\tg$ be their average
\STATE \textbf{Return}:  
$ (\tg_N, N)$
\end{algorithmic}
\end{algorithm}

\subsection{Lazy Stochastic Gradient Descent }~\label{sec:Stochastic Lazy SGD}
\textbf{Data Dependent Minibatch sizes:} 
The Lazy SGD (Alg.~\ref{algorithm:SAdaNGD2_new}) algorithm that we present in this section, uses a minibatch size that changes in between query points. Given a query point $x_s$, Lazy SGD invokes the noisy gradient oracle $\tO(1/\|g_s\|^2)$ times, where $g_s: =\nabla f(x_s)$~\footnote{Note that the  gradient norm, $\|g_s\|$, is unknown to the algorithm. Nevertheless it is estimated on the fly.}. Thus, in contrast to SGD which utilizes a fixed number of oracle calls per query point, our algorithm tends to stall in points with smaller gradients,  hence the name Lazy SGD.

Here we give some intuition regarding our adaptive minibatch size rule:
Consider the stochastic optimization setting.
However, imagine that instead of the noisy gradient oracle $\G$,  we may access an improved (imaginary) oracle which provides us with  unbiased  estimates, $\tg(x)$, that are accurate up to some \emph{multiplicative factor}, e.g.,
$$ \E[\tg(x)] = \nabla f(x), \; \text{and } \frac{1}{2}\|\nabla f(x) \| \leq \|\tg(x)\| \leq 2\|\nabla f(x)\|~.$$
Then intuitively we could have used these estimates instead of the exact normalized gradients
inside our (SC-)$\text{AdaNGD}_2$ algorithms  (Algs.~\ref{algorithm:AdaNGD},~\ref{algorithm:SC-AdaNGD} with $k=2$), and 
  still get similar (in expectation) data dependent bounds. 
Quite nicely, we may use our original  noisy oracle $\G$ to generate estimates from this imaginary oracle. This can be done by invoking $\G$  for $\tO(1/\|g_s\|^2)$ times at each query point.
Using this minibatch rule, the total number of calls to $\G$ (along all iterations) is equal to $T=\sum_{s=1}^S 1/\|g_s\|^2$. Plugging this into the data dependent bounds of $\text{(SC-)AdaNGD}_2$ (Thms.~\ref{thm:AdaNGD2}, \ref{thm:SC-AdaNGD2}), we get the well known 
$\tO(1/\sqrt{T})$/$\tO(1/T)$ rates for the stochastic convex settings.

\paragraph{The imaginary oracle:} 
The construction of the imaginary oracle  from the original oracle appears in Algorithm~\ref{algorithm:AE_new} (AE procedure) . 
It receives as an input,  $\G$, a generator of independent random vectors with an (unknown) expected  value $g\in\reals^d$.
The algorithm outputs two variables: $N$ which is an estimate of $1/\|g\|^2$, and $\tg_N$ an average of $N$ random vectors from $\G$. 
Thus, it is natural to think of $N\tg_N$ as an estimate for $g/\|g\|^2$. Moreover, it can be shown that $E[N(\tg_N- g)]=0$. Thus in a sense we receive an unbiased estimate.
The guarantees of Algorithm~\ref{algorithm:AE_new}
 appear in the following lemma  (see  proof in Appendix~\ref{app:AE}),
%The following lemma formalizes the above intuition regarding the construction of the imaginary oracle from the oracle at hand.
\begin{lemma}[Informal] \label{lem:AE_guaranteeNewInformal}
Let $T_{\max}\geq 1,\delta\in(0,1)$.
Suppose an oracle $\G:\K \mapsto \reals^d$ that generates $G$-bounded i.i.d. random vectors with an (unknown) expected value $g \in \reals^d$. 
Then  w.p.$\geq 1-\delta$, invoking  AE (Algorithm~\ref{algorithm:AE_new}), with 
$m_0 = \Theta(G\log(1/\delta))$, it is ensured that:
$$N =\Theta(\min\{ m_0/\|g\|^2, T_{\max} \} ),\; \text{and  }\; E[N(\tg_N- g)]=0~.$$
\end{lemma}

\paragraph{Lazy SGD:} 
Now, plugging the output of the AE algorithm into our offline algorithms $\text{(SC-)AdaNGD}_2$, we get their stochastic variants which appears in Algorithm~\ref{algorithm:SAdaNGD2_new} (Lazy SGD).
%This algorithm resembles SGD with the difference that  it stalls in query points with small gradients and we therefore denote it by the name LazySGD.  
This algorithm is equivalent to the offline version of $\text{(SC-)AdaNGD}_2$, with the difference that we use $n_s$ instead of $1/\|\nabla f(x_s)\|^2$ and $n_s\tg_s$ instead of $\nabla f(x_s)/\| \nabla f(x_s)\|^2$.

Let $T$ be a bound on the \emph{total number of queries} to the the first order oracle $\G$, and $\delta$ be the confidence parameter used to set $m_0$ in the AE procedure. The next lemmas present the guarantees of LazySGD for the stochastic settings
\footnote{We provide in expectation bounds. Note that we believe that high-probability bounds could be established. We leave this for future work.}  (see  proofs in Appendix~\ref{app:lazysgd1}, and~\ref{app:lazysgd2}),

\begin{lemma} \label{lem:LazySGD_general_expectation}
Let $\delta =O(1/T^{3/2})$, also 
let $\K$ be a convex set with diameter $D$, and $f$ be a convex function; and assume that $\| \G(x)\|\leq G$ w.p.$1$. Then using LazySGD (Algorithm~\ref{algorithm:SAdaNGD2_new}) with  $\eta_0 = D/\sqrt{2}G$, $p=1/2$, ensures that:
\begin{align*}
\E[f( \bar{x}_T)]- \min_{x\in\K}f(x) 
\le
 O\left(\frac{GD\log(T)}{\sqrt{T}} \right)~.
\end{align*}
\end{lemma}
\begin{lemma} \label{lem:LazySGDstronglyConvex_expectation}
Let $\delta =O(1/T^2)$, also 
let $\K$ be a convex set, and $f$ be an $H$-strongly-convex convex function; and assume that $\| \G(x)\|\leq G$ w.p.$1$. Then using LazySGD (Algorithm~\ref{algorithm:SAdaNGD2_new}) with  $\eta_0 = 1/H$, $p=1$ ensures that:
\begin{align*}
\E[f( \bar{x}_T)]- \min_{x\in\K}f(x) 
\le
 O\left(\frac{G^2\log^2(T)}{HT} \right)~.
\end{align*}
\end{lemma}
Note that LazySGD uses minibatch sizes  that are adapted to the magnitude of the gradients, and still maintains the optimal 
$O(1/\sqrt{T})/O(1/T)$ rates. In contrast using a fixed minibatch size $b$ for SGD might degrade the convergence rates, yielding
 $O(\sqrt{b}/\sqrt{T})/O(b/T)$ guarantees. 
This property of LazySGD may be beneficial when    considering distributed computations (see e.g.~\cite{dekel2012optimal}).

\section{Extensions} \label{sec:Extensions}
\textbf{Acceleration:} The catalyst approach, \cite{lin2015universal}, enables to take any first order method 
that ensures linear convergence rates in the strongly-convex and smooth case and transform it into an accelerated method  obtaining $O(\exp(-\sqrt{\gamma}T))$ rate in the strongly-convex and smooth case, and $O(1/T^2)$ rate in the smooth case. In particular, this acceleration applies to our $\text{SC-AdaNGD}_2$ Algorithm. 
Unfortunately, the catalyst approach requires the smoothness parameter, and the resulting accelerated $\text{SC-AdaNGD}_2$ is no longer universal.

\textbf{Other  online adaptive schemes:}
The adaptive methods that we have presented so far lean on AdaGrad (Alg.~\ref{algorithm:AdaGrad}).
Nevertheless, we may base our methods on other online algorithms with adaptive regret guarantees, and obtain convergence rates of the form, 
$$f(\bar{x}_T) - \min_{x\in\K} f(x) \le \frac{\R^{\A}\left(g_1/\|g_1\|^k,\ldots,g_T/\|g_T\|^k \right)}{\sum_{t=1}^T \|g_t\|^{k} }~,$$
where $\R^{\A}(\theta_1,\ldots,\theta_T)$ is the regret bound of algorithm $\A$ with respect to the linear loss sequence
$\{  \theta_t^\top x\}_{t=1}^T$.
For example we can  use the very popular version of AdaGrad, which employs a separate learning rate to different directions.
Also noteworthy  is the Multiplicative Weights (MW)  online algorithm, which over the simplex ensures a regret bound of the form (see~\cite{hazan2012linear},~\cite{clarkson2012sublinear}),
$$\R^{MW} \le \sqrt{\sum_{t=1}^T \|g_t\|_{\infty}^2 \log(d)}~.$$
Using $\text{AdaNGD}_k$ with the appropriate modifications: AdaGrad$\leftrightarrow$MW, and $\ell_2\leftrightarrow \ell_\infty$, yields similar adaptive guarantees as in Theorems~\ref{thm:AdaNGDnonsmooth},~\ref{thm:AdaNGD2},  with the difference that, $D \leftrightarrow \log d$, and $\ell_2 \leftrightarrow \ell_\infty $.

\section{Experiments}\label{sec:experiments}
As a preliminary experimental investigation we compare our $\text{SC-AdaNGD}_{k}$  to GD accelerated-GD, and line-search for two strongly-convex objectives\footnote{Line-search may invoke the gradient oracle several times in each iteration. To make a fair comparison, we present performance vs. $\#$calls to the gradient oracle}.
Concretely, we compare the above methods for the following quadratic (smooth) minimization problem,
$$\min_{x\in\reals^d} R(x) ~: =~ \frac{1}{2}\sum_{i=1}^d i\cdot x_i^2~. $$
,and also for the following non-smooth problem,
 $$\min_{\|x\|\leq 1} F(x) ~: =~ \frac{1}{2}\sum_{i=1}^d i\cdot x_i^2+\|x\|_1~. $$
where $x_i$ is the $i$'th component of $x$, and $\|x\|_1$ is the $\ell_1$ norm.
Note that both $R$ and $F$ are $1$-strongly-convex, however $R$ is $d$-smooth while $F$ is non-smooth.
Also, for both $R$ and $F$ the unique global minimum is in $x=0$.
We  initialize all of the methods at the same random point, and take $d=100$.

%%%
%%%%%%%%%%%%%%%%%%%%%%%%%%%%%%%%%
%%%%%%%%%%%%%%%%%%%%%%%%%%%%%%%%%
\begin{figure*}[t]
\centering
\subfigure[]{ \label{fig:smooth}
\includegraphics[trim = 10mm 60mm 20mm 75mm, clip,
width=0.32\textwidth ]{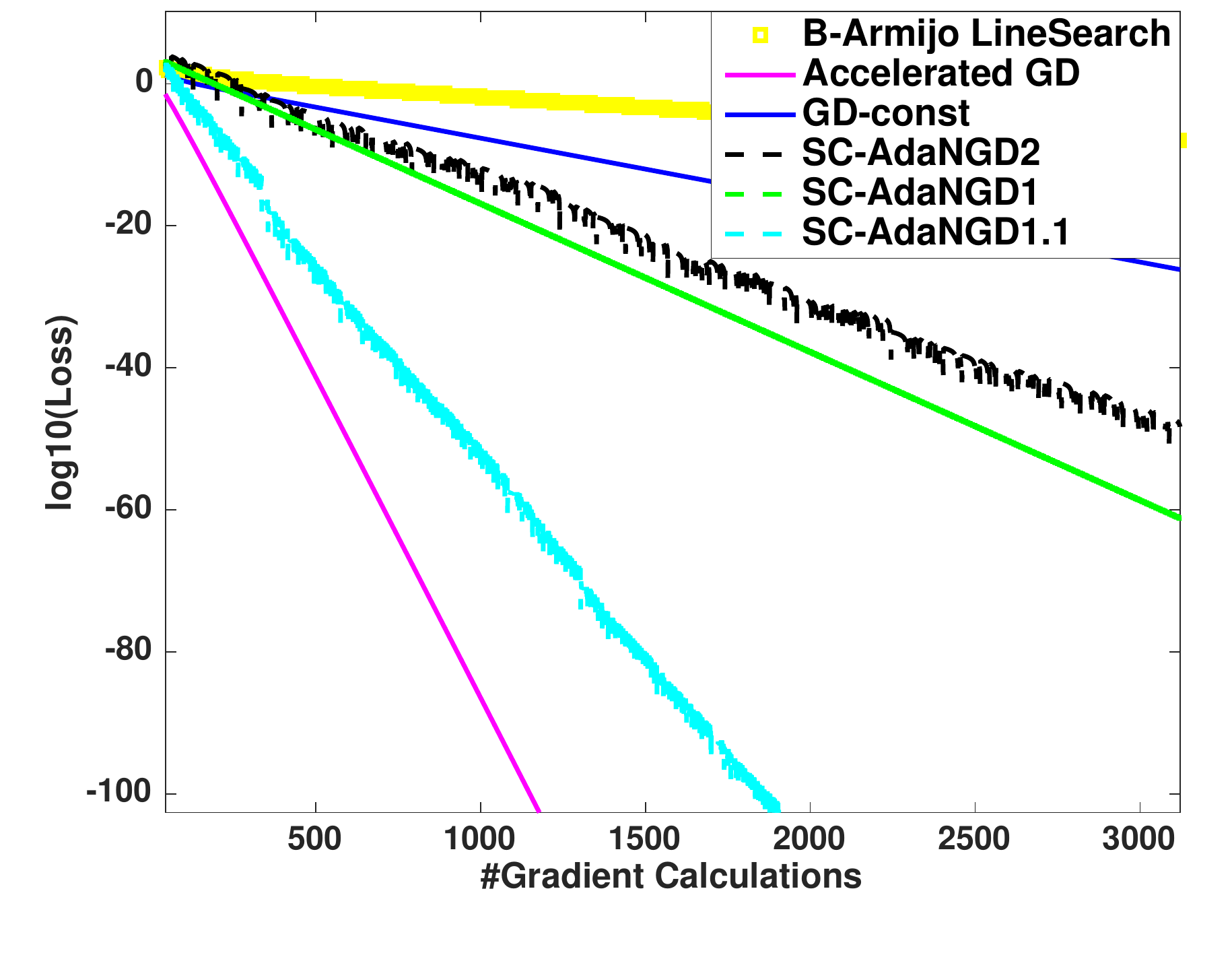}}
\subfigure[]{\label{fig:nonsmooth}
 \includegraphics[trim = 10mm 60mm 20mm 65mm, clip,
width=0.32\textwidth ]{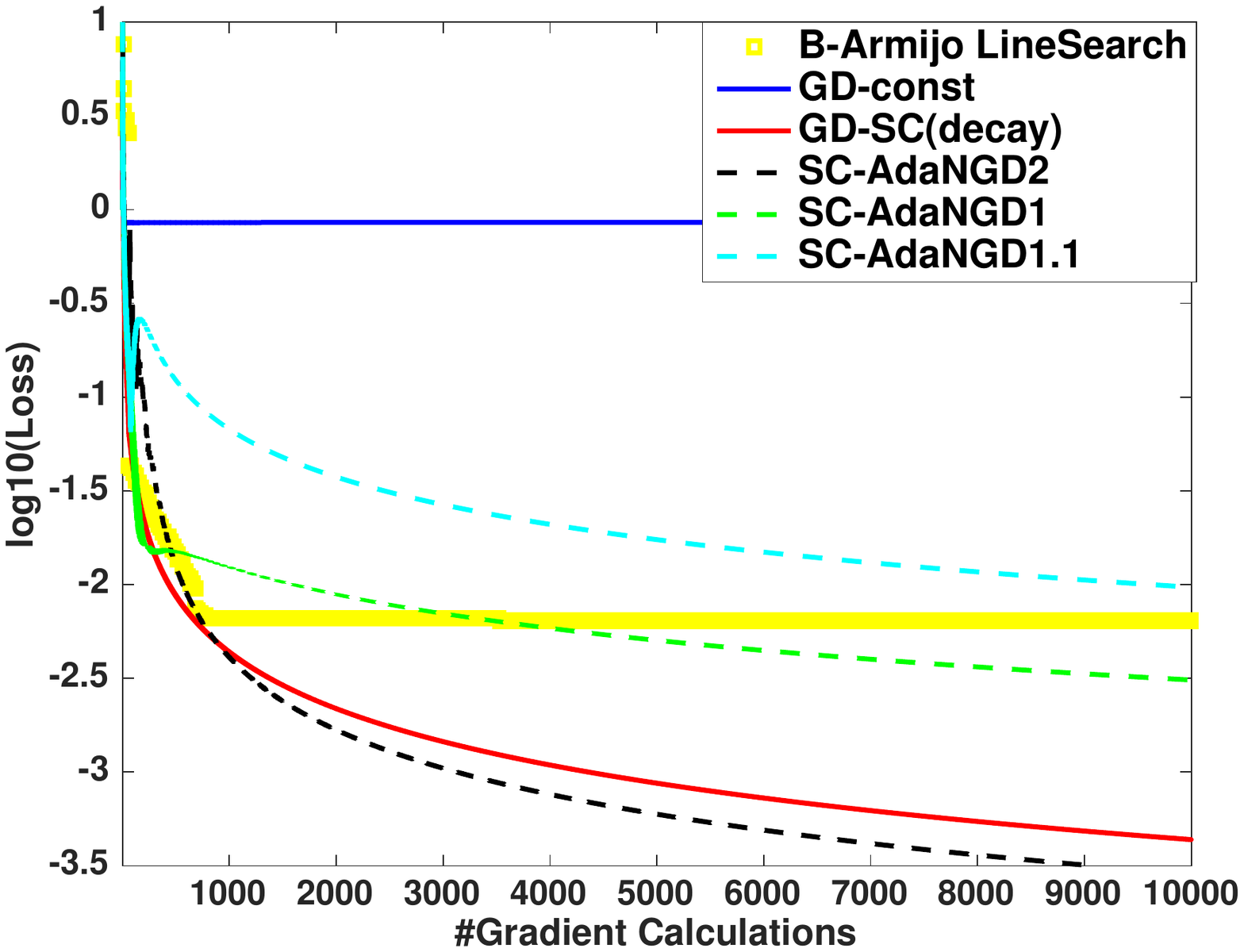}}
\subfigure[]{ \label{fig:2d}
 \includegraphics[trim = 10mm 60mm 20mm 65mm, clip,
width=0.32\textwidth ]{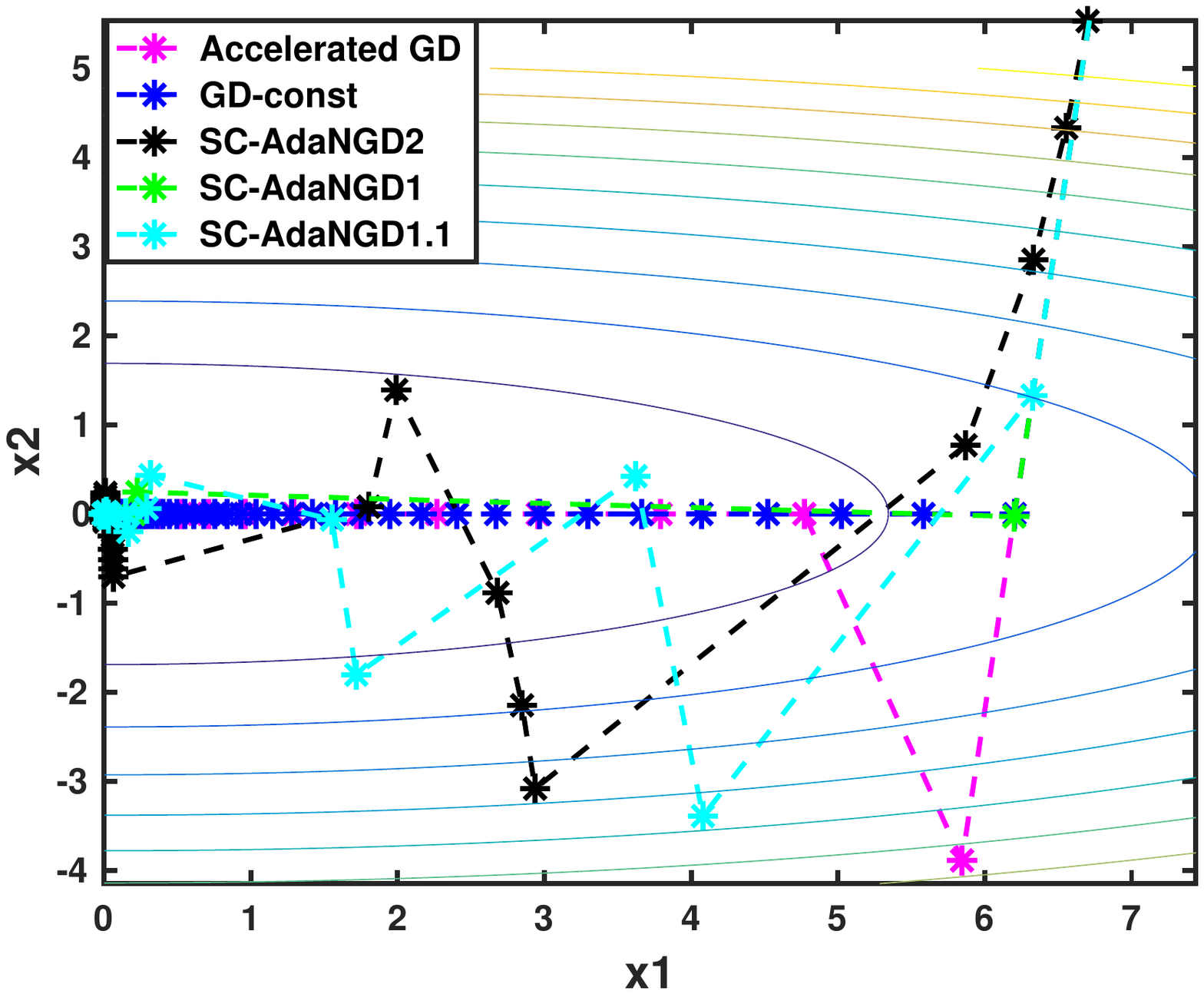}}
\caption{$\text{SC-AdaNGD}_k$ compared to  GD, accelerated-GD and line-search. Left: strongly-convex and smooth objective, $R(\cdot)$. Middle: strongly-convex and non-smooth objective, $F(\cdot)$. Right: iterates of these methods for a 2D quadratic objective, $Z(\cdot)$.} 
\label{fig:Exps}
\end{figure*}
%%%%%%%%%%%%%%%%%%%%%%%%%%%%%%%%%%
%%%%%%%%%%%%%%%%%%%%%%%%%%%%%%%%%

The results are depicted in Fig.~\ref{fig:Exps}. In Fig.~\ref{fig:smooth} we present our results for the smooth quadratic objective $R$.
We compare three $\text{SC-AdaNGD}_{k}$ variants  $k\in\{1,1.1,2 \}$, to GD which uses a constant learning rate $\eta_t = 1/\beta$ (recall $\beta=d=100$), and to Nesterov's accelerated method. While this is not surprising that the latter demonstrates the best performance, it is surprising that all $\text{SC-AdaNGD}_{k}$ variants are performing better than  GD/lines-search, and the   $k=1.1$ variant substantially outperforms GD. 
Also, in contrast to GD, $\text{SC-AdaNGD}_{k}$ are not descent methods, in the sense that the losses are not necessarily monotonically decreasing from one iteration to another.

Fig.~\ref{fig:nonsmooth} shows  the results for the non-smooth objective $F$, where we compare two $\text{SC-AdaNGD}_{k}$ variants  $k\in\{1,2 \}$, with two variants of 
GD, \textbf{(i)} const learning rate $\eta_t =1/\beta$, and \textbf{(ii)} decaying learning rate $\eta_t = 1/Ht$. We have also compared to accelerated-GD and found its performance to be similar to GD-const (and therefore omitted). As can be seen, GD with a constant learning rate is doing very poorly,  $\text{SC-AdaNGD}_{2}$ demonstrates the best performance,  and GD-SC (decay) lags behind only  by little. Note that for GD-SC (decay) we present results for a moving average over the GD iterates (which improve its performance).

The universality of $\text{SC-AdaNGD}_{k}$ for $k\in\{1,2\}$ is clearly evident from Figures~\ref{fig:smooth} ,\ref{fig:nonsmooth}.
In order to learn more about the character of $\text{SC-AdaNGD}$, we have applied the above methods to  
a simple 2D quadratic objective,
$$Z(x) = x_1^2 + 10x_2^2~.$$
The progress (iterates) of these methods is presented  in Fig.~\ref{fig:2d}.
It can be seen that GD and accelerated-GD converge quickly to the $x_1$ axis and  progress  along it towards $(0,0)$.
Conversely,  $\text{SC-AdaNGD}$ methods  progress diagonally, however take larger steps in the $x_1$ directions compared to GD and accelerated-GD.

\textbf{Robustness:}
We have also examined the robustness of $\text{SC-AdaNGD}$ compared to GD, accelerated-GD and line-search.
We applied these methods  to the quadratic objective $R$, however instead of the exact gradients we provided them with a slightly noisy and (unbiased) gradient feedback. The results when using noise perturbation magnitude of $10^{-6}$ appear in Fig.~\ref{fig:robustExp}. This behaviour persisted when we employed other noise magnitudes. 

%%%%%%%%%%%%%%%%%%%%%%%%%%%%%%%%%
%%%%%%%%%%%%%%%%%%%%%%%%%%%%%%%%%
\begin{figure*}[t]
\centering
\subfigure[]{ 
\includegraphics[trim = 10mm 60mm 20mm 75mm, clip,
width=0.4\textwidth ]{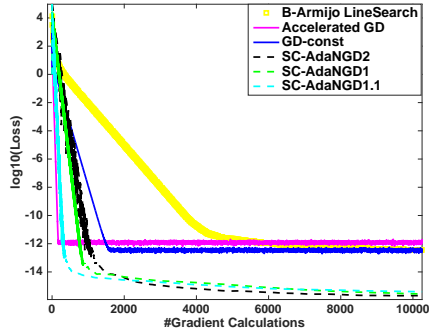}}
\caption{Robustness experiments comparing $\text{SC-AdaNGD}_k$ with  GD and accelerated-GD for the strongly-convex and smooth objective, $R(\cdot)$. Gradient oracle is perturbed with $\propto 10^{-6}$ noise magnitude.} 
\label{fig:robustExp}
\end{figure*}
%%%%%%%%%%%%%%%%%%%%%%%%%%%%%%%%%%
%%%%%%%%%%%%%%%%%%%%%%%%%%%%%%%%%

\textbf{Stochastic setting:}
We made a few experiments in the stochastic setting. While examining LazySGD,
we have found out that using the $n_s$ output of the AE procedure (Alg.~\ref{algorithm:AE_new}) is a too crude estimate for $1/\|g_s\|^2$ (due to the doubling procedure), which lead to unsatisfactory performance. Instead, we found that using $1/\|\tg_s\|^2$ is a much better approximation, that works very well in practice.

An initial experimental study  on several simple stochastic problems shows that LazySGD (with the above modification) compares with minibatch SGD, for various values of minibatch sizes. 
%As expected  LazySGD tends to increase minibatch size when approaching the optimum, and use smaller minibatch sizes elsewhere.
A more elaborate examination of LazySGD is left for future work.

%%%%Accelerated GD
%Y. Nesterov. Introductory lectures on convex optimization: A basic course.
%Kluwer Academic Publishers, 2004a.
%%%%%

\section{Discussion} \label{sec:Discussion}
We have presented a new approach which exhibits universality and new adaptive bounds in the offline convex optimization setting, and  provides a principled approach towards  minibatch size selection in the 
stochastic setting. 

Among the many questions that remain open is whether we can devise  ``accelerated"
universal methods.  Furthermore, our universality results only apply when the global minimum is inside the constraints. 
Thus, it is natural to seek for methods that ensure universality when this assumption is violated.  
Moreover, our algorithms depend on a parameter $k\in\reals$, but only the cases where $k\in\{0,1,2\}$ are well understood. Investigating a wider spectrum of $k$ values is intriguing. Lastly, it is interesting to modify and test our methods in non-convex scenarios, especially in the context of deep-learning applications.

\section*{Acknowledgement}
I would like to thank  Elad Hazan and Shai Shalev-Shwartz  for fruitful discussions during the early stages of this work.

This work was supported by the ETH Z\"urich Postdoctoral Fellowship and Marie Curie Actions for People COFUND program.

%\kl{Think if you want to mention this}\\
%In the stochastic setting, we conjecture that LazySGD  is more appropriate than SGD for distributed settings. It is appealing to
%explore in this direction from both a practical  and  theoretical perspective.

%\newpage
%\section*{Acknowledgement}
%I would like to thank Elad Hazan and  Shai Shalev-Shwartz for useful discussions during the early stages of this work.

%%%Bib
%%%%%%%%%%%%%%%%%%%%%%%%%%%%%%%%%%%%%%%%%%%%%%%%%%%%%%%%%%%%%%
%%%%%%%%%%%%%%%%%%%%%%%%%%%%%%%%%%%%%%%%%%%%%%%%%%%%%%%%%%%%%%
\bibliography{bib}
\bibliographystyle{icml2017}

%%%Appendix
%%%%%%%%%%%%%%%%%%%%%%%%%%%%%%%%%%%%%%%%%%%%%%%%%%%%%%%%%%%%%%
%%%%%%%%%%%%%%%%%%%%%%%%%%%%%%%%%%%%%%%%%%%%%%%%%%%%%%%%%%%%%%
%%%%%%%%%%%%%%%%%%%%%%%%%%%%%%%%%%%%%%%%%%%%%%%%%%%%%%%%%%%%%%
%%%%%%%%%%%%%%%%%%%%%%%%%%%%%%%%%%%%%%%%%%%%%%%%%%%%%%%%%%%%%%%%%%%%%%%%%%%%%%%%%%%%%%%%%%%%%%%%%%%%%%%%%%%%%%%%%%%%%%%%%%%%%%%%%%%%%%%%%%%%%%%%%%%%%%%%%%%%%%%%%%%%%%%%%%%%%%%%%%%%%%%%%%
\newpage
%\normalsize
\onecolumn

\appendix
\section{Proofs for Section~\ref{sec:AdaNGD_general} (AdaNGD)}
%%%AdaGrad Proof
%%%%%%%%%%%%%%%%%%%%%%%%%%%%%%%%%%%%%%%%%%%%%%%%%%%%%%%%%%%%%%
%%%%%%%%%%%%%%%%%%%%%%%%%%%%%%%%%%%%%%%%%%%%%%%%%%%%%%%%%%%%%%
\subsection{Proof of Theorem~\ref{theorem:AdaGrad} (AdaGrad)}
\begin{proof}
Let $x\in\K$ and Consider the update rule $x_{t+1}=\Pi_\K(x_t - \eta_t g_t)$. We can write:
\begin{align*}
\|x_{t+1}-x \|^2 \leq \|x_t-x \|^2 -2\eta_t g_t(x_t-x) + \eta_t^2\| g_t\|^2
\end{align*}  
Re-arranging the above we get: 
$$g_t(x_t-x)\leq \frac{1}{2\eta_t}\left(\|x_t-x \|^2-\|x_{t+1}-x \|^2\right) +\frac{\eta_t}{2}\| g_t\|^2~.$$
Combined with the convexity of $f_t$ and summing over all rounds we conclude that $\forall x\in\K$,
\begin{align*}
\sum_{t=1}^T f_t(x_t)-\sum_{t=1}^Tf_t(x)&\leq \sum_{t=1}^T \frac{\|x_t-x \|^2}{2}
\left(\frac{1}{\eta_t}-\frac{1}{\eta_{t-1}}\right)
+\sum_{t=1}^T\frac{\eta_t}{2}\| g_t\|^2 \\
&\leq \frac{D^2}{2}\sum_{t=1}^T\left(\frac{1}{\eta_t}-\frac{1}{\eta_{t-1}} \right)
+\frac{D}{2\sqrt{2}}\sum_{t=1}^T\frac{\| g_t\|^2}{\sqrt{\sum_{\tau=1}^t \|g_\tau \|^2}} \\
&\leq \frac{D}{2}\sqrt{2\sum_{t=1}^T \|g_t\|^2}  
+\frac{D}{\sqrt{2}}\sqrt{\sum_{t=1}^T \|g_t\|^2}  \\
& = \sqrt{2D^2\sum_{t=1}^T \|g_t\|^2} 
\end{align*}
here in the first inequality we denote $\eta_0 =\infty$, 
the second inequality uses $\text{diam}\K=D$ and $\eta_t\leq \eta_{t-1}$, the third inequality uses the following lemma from \cite{mcmahan2010adaptive}: 
\begin{lemma}\label{lem:SqrtSum}
For any non-negative numbers $a_1,\ldots, a_n$ the following holds:
$$\sum_{i=1}^n \frac{a_i}{\sqrt{\sum_{j=1}^i a_j}} \leq 2\sqrt{\sum_{i=1}^n a_i}$$
\end{lemma}
\end{proof}

\subsection{Proof of Lemma~\ref{lem:AdaNGDnonsmooth_k}}
\begin{proof}%[Proof of Lemma~\ref{lem:AdaNGDnonsmooth_k}]
Notice that  $\text{AdaNGD}_k$  described in Algorithm~\ref{algorithm:AdaNGD}, is equivalent to applying 
AdaGrad (Algorithm~\ref{algorithm:AdaGrad}) to the following  sequence of linear loss functions: 
$$\left\{\tf_t(x) : = \frac{1}{\| g_t\|^k} g_t^\top x \right\}_{t=1}^T~.$$  
The regret bound of AdaGrad appearing in Theorem~\ref{theorem:AdaGrad} implies the following  for any $x\in\K$:
%Consider the following sequence of linear loss functions: $\{\frac{1}{\| g_t\|^k} g_t^\top x\}_{t=1}^T$. Now notice that
% $\text{AdaNGD}_k$ described above, is equivalent to AdaGrad ( applied  this sequence. The standard regret bound for AdaGrad (see Duchi, McMahann and Streeter Aberenthy), applied to the sequence $\{\frac{1}{\| g_t\|^k} g_t^\top x\}_{t=1}^T$, ensures
%\begin{lemma} \label{lem:AdaGradnonsmooth2}
%For any $x\in\K$, we have:
%\end{lemma}
\begin{align} \label{eq:RegAdangd_k}
\sum_{t=1}^T \frac{1}{\|Êg_t\|^k}g_t^\top( x_t-x)  \le \sqrt{2D^2\sum_{t=1}^T {1}/{\| g_t\|^{2(k-1)} }}~.
\end{align}
Using the above bound  together with Jensen's inequality,  enables to bound the excess loss of $\text{AdaNGD}_k$: 
\begin{align*}
f(\bar{x}_T)-f(x^*)&
\le
 \sum_{t=1}^T  \frac{ \|g_t\|^{-k}}{\sum_{\tau=1}^T \|g_\tau\|^{-k}} \big(f(x_t) -  f(x^*)\big)    \\%\quad \because \text{Jensen's inequlity} \\
& \le
  \sum_{t=1}^T  \frac{ \|g_t\|^{-k}}{\sum_{\tau=1}^T \|g_\tau\|^{-k}} g_t^\top(x_t-x^*) \\% \quad \because \text{gradient inequality} \\
& 
\eq
 \frac{1}{\sum_{\tau=1}^T \|g_\tau\|^{-k}}\sum_{t=1}^T  \frac{1}{\| g_t\|^k}g_t^\top(x_t-x^*) \\
&
\le
\frac{ \sqrt{2D^2\sum_{t=1}^T {1}/{\| g_t\|^{2(k-1)}}}}{\sum_{\tau=1}^T {1}/{\|g_\tau\|^{k}}}~,
\end{align*}
where the second line uses the gradient inequality. 
\end{proof}

\subsection{Proof of Theorem~\ref{thm:AdaNGDnonsmooth}} \label{app:adangd1}
%%%ADANGD1_SMOOTH PROOF
%%%%%%%%%%%%%%%%%%%%%%%%%%%%%%%%%%%%%%%%%%%%%%%%%%%%%%%%%%%%%%
%%%%%%%%%%%%%%%%%%%%%%%%%%%%%%%%%%%%%%%%%%%%%%%%%%%%%%%%%%%%%%
%%%%%%%%%%%%%
\begin{proof}%[Proof of Theorem~\ref{thm:AdaNGDnonsmooth}]
The data dependent bound,
\begin{align}\label{eq:ProofAdangd1}
f(\bar{x}_T)-f(x^*)&
\le
\frac{ \sqrt{2D^2 T}}{\sum_{t=1}^T {1}/{\|g_t\|}}~,
\end{align}
 is a direct corollary of Lemma~\ref{lem:AdaNGDnonsmooth_k} with $k=1$.
 Note that the above bound holds for both smooth/non-smooth cases.
The general case bound holds directly by using $\|g_t\|\leq G$.

Next we focus on the second part of the theorem regarding the smooth case.
We will first require the following lemma regarding smooth objectives,
\begin{lemma} \label{lemma:GradIneqSmooth}
Let $F:\reals^d \mapsto \reals$ be a $\beta$-smooth function, and let $x^* =\argmin_{x\in \reals^d}F(x)$, then,
$$ \| \nabla F(x)\|^2 \le 2\beta \left( F(x) - F(x^*)\right), \quad \forall x\in \reals^d~.$$
\end{lemma} 

The above lemma enables to upper bound  sum of gradient norms in the query points of $\text{AdaNGD}_1$,
%First, notice the following:
\begin{align} \label{eq:Sum_norms_bound}
\sum_{t=1}^T\|g_t\| 
&\eq 
\sum_{t=1}^T\frac{\|g_t\|^2}{\|g_t\|} \nonumber\\
&\le 
\sum_{t=1}^T \frac{{2\beta}}{\|g_t\|}\big(f(x_t)-f(x^*)\big)  \nonumber\\
& \le
\sum_{t=1}^T \frac{{2\beta}}{\|g_t\|}g_t^\top(x_t-x^*)  \nonumber\\
& \eq
2\beta \sum_{t=1}^T \hat{g}_t^\top(x_t-x^*)   \nonumber\\
&\le
2\sqrt{2}\beta D\sqrt{T}~, 
\end{align}
where the last line follows by the regret guarantee of AdaGrad for the following sequence (see Equation~\eqref{eq:RegAdangd_k}),
$$\left\{\tf_t(x): =\frac{1}{\| g_t\|} g_t^\top x\right\}_{t=1}^T~.$$
The second line is a consequence of Lemma~\ref{lemma:GradIneqSmooth} regarding smooth objectives.
%\begin{lemma} \label{lemma:GradIneqSmooth11}
%Let $F:\reals^d \mapsto \reals$ be a $\beta$-smooth function, and let $x^* =\argmin_{x\in \reals^d}F(x)$, then,
%$$ \| \nabla F(x)\|^2 \le 2\beta \left( F(x) - F(x^*)\right), \quad \forall x\in \reals^d~.$$
%\end{lemma} 
Now utilizing the convexity of the function $H(z) =1/z$ for $z>0$, and applying Equation~\eqref{eq:Sum_norms_bound},
we may bound the sum of inverse gradients:
\begin{align*}
\sum_{\tau=1}^T \frac{1}{\|g_\tau\|}
& \eq
 T\frac{1}{T}\sum_{\tau=1}^T \frac{1}{\|g_\tau\|}  
 \ge
  T\frac{1}{\frac{1}{T}\sum_{\tau=1}^T \|g_\tau\| }\\
&\ge
  T\frac{1}{2\sqrt{2}\beta D/\sqrt{T}}~.
\end{align*}
Rearranging the latter equation, and using  Equation~\eqref{eq:ProofAdangd1} concludes the proof,
$$f( \bar{x}_T)- \min_{x\in\K}f(x) \le \frac{D\sqrt{2T}}{\sum_{\tau=1}^T 1/\|g_\tau\|} \le  \frac{4\beta D^2}{T}~.$$
\end{proof}

\subsection{Proof of Theorem~\ref{thm:AdaNGD2}}\label{app:adangd2}
%%%ADANGD2_SMOOTH PROOF
%%%%%%%%%%%%%%%%%%%%%%%%%%%%%%%%%%%%%%%%%%%%%%%%%%%%%%%%%%%%%%
%%%%%%%%%%%%%%%%%%%%%%%%%%%%%%%%%%%%%%%%%%%%%%%%%%%%%%%%%%%%%%
%%%%%%%%%%%%%%%%%%%
\begin{proof}%[Proof of Theorem~\ref{thm:AdaNGD2}]
The data dependent bound,
\begin{align}\label{eq:ProofAdangd2}
f(\bar{x}_T)-f(x^*)&
\le
\frac{ \sqrt{2D^2 }}{\sqrt{\sum_{t=1}^T {1}/{\|g_t\|^2}}}~,
\end{align}
 is a direct corollary of Lemma~\ref{lem:AdaNGDnonsmooth_k} with $k=2$.
  Note that the above bound holds for both smooth/non-smooth cases.
The general case bound holds directly by using $\|g_t\|\leq G$.

We will now focus on the second part of the theorem regarding the smooth case.
Let us lower bound $\sum_{t=1}^T1/\|g_t\|^2$ for $\text{AdaNGD}_2$:
%First, notice the following:
\begin{align} \label{eq:Sum_norms_bound_2}
T
&\eq 
\sum_{t=1}^T\frac{\|g_t\|^2}{\|g_t\|^2} \nonumber\\
&\le 
\sum_{t=1}^T \frac{{2\beta}}{\|g_t\|^2}\big(f(x_t)-f(x^*)\big)  \nonumber\\
& \le
\sum_{t=1}^T \frac{{2\beta}}{\|g_t\|^2}g_t^\top(x_t-x^*)  \nonumber\\
& \eq
2\beta \sum_{t=1}^T \left(\tf_t(x_t) - \tf_t(x^*)  \right) \nonumber\\
&\le
2\sqrt{2}\beta D \sqrt{\sum_{t=1}^T \frac{1}{\|g_t\|^2}}~, 
\end{align}
where the last line follows by the regret guarantee of AdaGrad for the following sequence (see Equation~\eqref{eq:RegAdangd_k}),
$$\left\{\tf_t(x) =\frac{1}{\| g_t\|^2} g_t^\top x \right\}_{t=1}^T~.$$
The second line is a consequence of  Lemma~\ref{lemma:GradIneqSmooth}. 
Combining Equation~\eqref{eq:Sum_norms_bound_2} together with  
Equation~\eqref{eq:ProofAdangd2} concludes the proof.
\end{proof}

%%%Smoothness-Gradients Proof
%%%%%%%%%%%%%%%%%%%%%%%%%%%%%%%%%%%%%%%%%%%%%%%%%%%%%%%%%%%%%%
%%%%%%%%%%%%%%%%%%%%%%%%%%%%%%%%%%%%%%%%%%%%%%%%%%%%%%%%%%%%%%
\subsection{Proof of Lemma~\ref{lemma:GradIneqSmooth}}
\begin{proof}
The $\beta$ smoothness of $F$ means the following to hold $\forall x,u\in\reals^d$,
$$F(x+u) \leq F(x) +\nabla F(x)^\top u+\frac{\beta}{2}\|u\|^2 ~.$$
Taking  $u=-\frac{1}{\beta}\nabla F(x)$ we get,
$$F(x+u) \le F(x) -\frac{1}{\beta}\|\nabla F(x)\|^2+\frac{1}{2\beta}\|\nabla F(x)\|^2~.$$
Thus:
\begin{align*}
\|\nabla F(x)\| &\le \sqrt{2\beta \big( F(x) -F(x+u)\big)}\\
&  \le  \sqrt{2\beta \big(F(x) -F(x^*)\big)}~,
\end{align*}
where in the last inequality we used $F(x^*) \leq F(x+u)$ which holds since $x^*$ is the \emph{global} minimum.
\end{proof}

%%%Strongly convex Section
%%%%%%%%%%%%%%%%%%%%%%%%%%%%%%%%%%%%%%%%%%%%%%%%%%%%%%%%%%%%%%
%%%%%%%%%%%%%%%%%%%%%%%%%%%%%%%%%%%%%%%%%%%%%%%%%%%%%%%%%%%%%%

%%%Proof SC-AdaNGD_k
%%%%%%%%%%%%%%%%%%%%%%%%%%%%%%%%%%%%%%%
%%%%%%%%%%%%%%%%%%%%%%%%%%%%%%%%%%%%%%%
\section{Proofs for Section~\ref{sec:AdaNGD_StronglyConvex} (SC-AdaNGD)}
\subsection{Proof of Lemma~\ref{lem:SC-AdaNGDnonsmooth2}} \label{app:scadangd_General}
%%%SC-AdaGrad
%%%%%%%%%%%%%%%%%%%%%%%%%%%%%%%%%%%%%%%%%%%%%%%%%%%%%%%%%%%%%%
\begin{algorithm}[t]
\caption{Strongly-Convex Adaptive  Gradient Descent ($\text{SC-AdaGrad}$) }
\label{algorithm:SC-AdaGrad}
\begin{algorithmic}
\STATE \textbf{Input}: \#Iterations $T$, $x_1\in \reals^d$, set $\K$ 
\STATE {Set}: $Q_0 =0 $
\FOR{$t=1 \ldots T$ }
\STATE {Calculate:} $g_t= \nabla f_t(x_t)$
\STATE {Let:} $H_t$ be the strong-convexity parameter of $f_t(\cdot)$
\STATE {Update:}
$$Q_{t} = Q_{t-1} + H_t $$
\STATE {Set} $\eta_t = 1/Q_t$
\STATE {Update:}
$$x_{t+1}= \Pi_{\K}\left( x_{t}-\eta_t {g}_{t}\right)$$
\ENDFOR
\end{algorithmic}
\end{algorithm}

\begin{proof}
We will require the following extension of Theorem $1$ from \cite{hazan2007logarithmic}. Its proof is provided in Section~\ref{sec:Proof_Lemma_ogd_strCvx}.
%\vspace{-40pt}
\begin{lemma}[SC-AdaGrad, Alg~\ref{algorithm:SC-AdaGrad}]\label{lemma:ogd_strCvx}
Assume that we receive a sequence of convex loss functions $f_t:\K\mapsto \reals,\; t\in[T]$, and suppose that
each function $f_t$ is $H_t$-strongly-convex. Using the update rule $x_{t+1} = \Pi_\K(x_t-\eta_t g_t)$ where $g_t=\nabla f_t(x_t)$ and $\eta_t =(\sum_{\tau=1}^t H_\tau)^{-1}$ yields the following regret bound:
\begin{align*}
\sum_{t=1}^T f_t(x_t)-\sum_{t=1}^Tf_t(x) \le \frac{1}{2}\sum_{t=1}^T\eta_t\|g_t\|^2~.
\end{align*}
\end{lemma}
We are now ready to go on with the proof.
%Notice that  $\text{AdaNGD}_k$  described in Algorithm~\ref{algorithm:AdaNGD}, is equivalent to applying 
%AdaGrad (Algorithm~\ref{algorithm:AdaGrad}) to the following sequence of sequence of linear loss functions: $\{\frac{1}{\| g_t\|^k} g_t^\top x\}_{t=1}^T$.  The regret bound of AdaGrad appearing in Theorem~\ref{theorem:AdaGrad} implies the following  for any $x\in\K$:
Note that $\text{SC-AdaNGD}_k$  depicted in Algorithm~\ref{algorithm:SC-AdaNGD} is equivalent to performing SC-AdaGrad updates $x_{t+1}=\Pi_\K(x_t-\eta_t \nabla \tf_t(x_t))$  over the following loss sequence:
 $$\left\{\tf_t(x) = \frac{1}{\|g_t\|^k}g_t^\top x+\frac{H}{2\|g_t \|^k}\|x-x_t\|^2 \right\}_{t=1}^T$$ 
 where $g_t =\nabla f_t(x_t)$. Note that each $\tf_t(x)$ is $\frac{H}{\| g_t\|^k}$-strongly-convex, and that the learning rate is inversely proportional to the cumulative sum of strong-convexities. Thus Lemma~\ref{lemma:ogd_strCvx} implies the following to hold for any $x\in\K$:
\begin{align*} %\label{eq:RegretStronglycVX}
\sum_{t=1}^T \tf_t(x_t)-\sum_{t=1}^T\tf_t(x) 
\le 
\frac{1}{2H}\sum_{t=1}^T \frac{\| g_t\|^{-2(k-1)}}{\sum_{\tau=1}^t \|g_\tau \|^{-k}}~.
\end{align*}
Combining the latter bound with the definition of $\bar{x}_T$, and applying Jensen's inequality we conclude:
\begin{align*}
f(\bar{x}_T)-f(x^*)
&\le
 \sum_{t=1}^T  \frac{ \|g_t\|^{-k}}{\sum_{\tau=1}^T \|g_\tau\|^{-k}} \big(f(x_t) -  f(x^*)\big)    \\%\quad \because \text{Jensen's inequlity} \\
& \le   \frac{1}{\sum_{t=1}^T \|g_t\|^{-k}}\sum_{t=1}^T  \|g_t\|^{-k}\left(g_t^\top(x_t-x^*) - \frac{H}{2}\|x_t-x^*\|^2  \right) \\% \quad \because \text{gradient inequality} \\
& \eq
\frac{1}{\sum_{t=1}^T \|g_t\|^{-k}}\sum_{t=1}^T  \left(\tf_t(x_t)-\tf_t(x^*)  \right) \\% 
&  \le
 \frac{1}{2H\sum_{t=1}^T \|g_t\|^{-k}}\sum_{t=1}^T \frac{\| g_t\|^{-2(k-1)}}{\sum_{\tau=1}^t \|g_\tau \|^{-k}}~,
\end{align*}
where we used the $H$-strong-convexity of $f$ in the second line.
\end{proof}

%%%Proof SC-AdaNGD_1
%%%%%%%%%%%%%%%%%%%%%%%%%%%%%%%%%%%%%%%
%%%%%%%%%%%%%%%%%%%%%%%%%%%%%%%%%
%\begin{proof}
\subsection{Proof of Theorem~\ref{thm:SC-AdaNGDnonsmooth}}\label{app:scadangd1}

\begin{proof}%[Proof of Theorem~\ref{thm:SC-AdaNGDnonsmooth}]

We will require the following lemma, its proof is provided in 
Section~\ref{sec:Proof_lem:Log_sum}.
\begin{lemma} \label{lem:Log_sum}
For any non-negative real numbers $a_1,\ldots, a_n\geq 1$,
\begin{align*}
\sum_{i=1}^n \frac{a_i}{\sum_{j=1}^i a_j} 
\le 
1+\log\left( \sum_{i=1}^n a_i\right) ~.
\end{align*}
\end{lemma}
\vspace{10pt}

Combining the above lemma together with Lemma~\ref{lem:SC-AdaNGDnonsmooth2} and using 
$k=1$, we obtain,
\begin{align*}
f(\bar{x}_T)-f(x^*)
&  \le
 \frac{1}{2H\sum_{t=1}^T \|g_t\|^{-1}}\sum_{t=1}^T \frac{1}{\sum_{\tau=1}^t \|g_\tau \|^{-1}}\\
 &  \le
 \frac{1}{2H\sum_{t=1}^T \|g_t\|^{-1}}\sum_{t=1}^T \frac{G\|g_t\|^{-1}}{\sum_{\tau=1}^t \|g_\tau \|^{-1}}\\
 &  \le
 \frac{G}{2H\sum_{t=1}^T \|g_t\|^{-1}}\sum_{t=1}^T \frac{G\|g_t\|^{-1}}{\sum_{\tau=1}^t G\|g_\tau \|^{-1}}\\
&  \le
 \frac{G}{2H\sum_{t=1}^T \|g_t\|^{-1}} \left(1+\log\left( \sum_{t=1}^T \frac{G}{\|g_t \|} \right)  \right)
\end{align*}
where the second line uses $\|g_t\|\leq G$, and the last line uses Lemma~\ref{lem:Log_sum}.
Note that the above bound holds for both smooth/non-smooth cases.

%\newpage
%\vspace{-30pt}
We now turn to prove the second part of the theorem regarding the smooth case.
First let us bound the sum of gradient norms in the query points of $\text{SC-AdaNGD}_1$:
%By Corollary~\ref{thm:SC-AdaNGDnonsmooth}, we already have,
%$$f( \bar{x}_T)-\min_{x\in\K}f(x) \leq \frac{G\log T}{2H \sum_{t=1}^T \frac{1}{\|g_t \|}}$$
%So we are left to prove the following:
%$$\frac{G\log T}{2H \sum_{t=1}^T \frac{1}{\|g_t \|}} \leq \frac{(\beta/H)G^2 \log^2 T}{T^2}$$
%First, notice the following:
\begin{align*}
\sum_{t=1}^T\|g_t\| 
&\eq
 \sum_{t=1}^T\frac{\|g_t\|^2}{\|g_t\|} \\
&\le 
\sum_{t=1}^T \frac{{2\beta}}{\|g_t\|}\big(f(x_t)-f(x^*)\big)  \\
& \le
 \sum_{t=1}^T \frac{{2\beta}}{\|g_t\|}\left( g_t^\top(x_t-x^*)-\frac{H}{2}\|x_t-x^* \|^2 \right)   \\
& \eq
 2\beta \sum_{t=1}^T \left(  \tf_t(x_t) - \tf_t(x^*)  \right)\\
&\le
\frac{\beta}{H} \sum_{t=1}^T \frac{1}{\sum_{\tau=1}^t \|g_\tau \|^{-1}}\\
&\le
\frac{\beta}{H} G \left(1+\log\left( \sum_{t=1}^T \frac{G}{\|g_t \|} \right)  \right)~,
\end{align*}
where the second line  uses Lemma~\ref{lemma:GradIneqSmooth}, the third line uses the strong-convexity of $f$, the fourth line uses the regret bound of the SC-AdaGrad algorithm   over the following sequence (see Equation~\eqref{eq:RegretStronglycVX}),
 $$\left\{\tf_t(x) = \frac{1}{\|g_t\|}g_t^\top x+\frac{H}{2\|g_t \|}\|x-x_t\|^2 \right\}_{t=1}^T~,$$ 
and the last line uses Lemma~\ref{lem:Log_sum}.
Combining the convexity of the function $H(z) =1/z$ for $z>0$, together with the above inequality, we may bound the sum of inverse gradient norms,
\begin{align*}
\sum_{\tau=1}^T \frac{1}{\|g_\tau\|}& 
\eq
 T\frac{1}{T}\sum_{\tau=1}^T \frac{1}{\|g_\tau\|}  
 \ge
  T\frac{1}{\frac{1}{T}\sum_{\tau=1}^T \|g_\tau\| }\\
&\ge
  T^2 \frac{1}{(\beta/H) G\left(1+\log\left( \sum_{t=1}^T \frac{G}{\|g_t \|} \right)  \right) }~.
\end{align*}
Rearranging the latter equation, and using  the data dependent bound for $\text{SC-AdaNGD}_1$ concludes the proof,
\begin{align*}
f( \bar{x}_T)- \min_{x\in\K}f(x) 
\le
\frac{(\beta/H)G^2 \left(1+\log T\right)^2}{HT^2}~.
\end{align*}

\end{proof}

\subsection{Proof of Theorem~\ref{thm:SC-AdaNGD2}}	\label{app:scadangd2}
				
\begin{proof}%[Proof of Theorem~\ref{thm:SC-AdaNGD2}]
The data dependent bound,
\begin{align}\label{eq:SC-AdaNGD2_proof}
f( \bar{x}_T)- \min_{x\in\K}f(x) 
&\le
 \frac{1+\log(G^2 \sum_{t=1}^T\|g_t\|^{-2})}{2H\sum_{t=1}^T \|g_t\|^{-2}} 
\end{align}
 is a direct corollary of Lemma~\ref{lem:SC-AdaNGDnonsmooth2}  with $k=2$, combined with Lemma~\ref{lem:Log_sum}.  Note that the above bound holds for both smooth/non-smooth cases.

We now turn to prove the second part of the theorem regarding the smooth case.
Let us lower bound  $\sum_{t=1}^T 1/\|g_t\|^2$, for $\text{SC-AdaNGD}_2$:
%By Corollary~\ref{thm:SC-AdaNGDnonsmooth}, we already have,
%$$f( \bar{x}_T)-\min_{x\in\K}f(x) \leq \frac{G\log T}{2H \sum_{t=1}^T \frac{1}{\|g_t \|}}$$
%So we are left to prove the following:
%$$\frac{G\log T}{2H \sum_{t=1}^T \frac{1}{\|g_t \|}} \leq \frac{(\beta/H)G^2 \log^2 T}{T^2}$$
%First, notice the following:
\begin{align} \label{eq:ExpRate}
T
&\eq
 \sum_{t=1}^T\frac{\|g_t\|^2}{\|g_t\|^2}    \nonumber \\
&\le 
\sum_{t=1}^T \frac{{2\beta}}{\|g_t\|^2}\big(f(x_t)-f(x^*)\big)  \nonumber \\
& \le
 \sum_{t=1}^T \frac{{2\beta}}{\|g_t\|^2}\left( g_t^\top(x_t-x^*)-\frac{H}{2}\|x_t-x^* \|^2 \right)  \nonumber \\
& \eq
 2\beta \sum_{t=1}^T \left(  \tf_t(x_t) - \tf_t(x^*)  \right)    \nonumber \\
&\le
\frac{\beta}{H} \sum_{t=1}^T \frac{\| g_t\|^{-2}}{\sum_{\tau=1}^t \|g_\tau \|^{-2}}\nonumber \\
&\le
\frac{\beta}{H}\left( 1+\log(G^2 \sum_{t=1}^T\|g_t\|^{-2}) \right)~,
\end{align}
where the second line  uses Lemma~\ref{lemma:GradIneqSmooth}, the third line uses the strong-convexity of $f$, the fifth line uses the regret bound of the SC-AdaGrad algorithm for the following sequence (see Equation~\eqref{eq:RegretStronglycVX}),
 $$\left\{\tf_t(x) = \frac{1}{\|g_t\|^2}g_t^\top x+\frac{H}{2\|g_t \|^2}\|x-x_t\|^2 \right\}_{t=1}^T~,$$ 
 and the last line uses Lemma~\ref{lem:Log_sum}.
Now Equation~\eqref{eq:ExpRate} implies,
\begin{align} \label{eq:Ineq333}
G^2\sum_{t=1}^T \|g_t\|^{-2} \ge \frac{1}{3} e^{\frac{H}{\beta}T}~.
\end{align}
Now let $z\in\reals$ and note that the function $A(z):=\frac{1+\log(z)}{z}$ is monotonically decreasing for $z\geq1$. Let $z= G^2\sum_t \|g_t\|^{-2}$ and assume $ \frac{1}{3} e^{\frac{H}{\beta}T} \geq 1$; combining this  with Equation~\eqref{eq:SC-AdaNGD2_proof},\eqref{eq:Ineq333}, concludes the proof.
Note that the case  $ \frac{1}{3} e^{\frac{H}{\beta}T} \leq 1$ is not too interesting.

\end{proof}

\subsection{Proof of Lemma~\ref{lemma:ogd_strCvx}} \label{sec:Proof_Lemma_ogd_strCvx}
\begin{proof}
Let $x\in\K$ and Consider the update rule $x_{t+1}=\Pi_\K(x_t - \eta_t g_t)$. We can write:
\begin{align*}
\|x_{t+1}-x \|^2 \le \|x_t-x \|^2 -2\eta_t g_t(x_t-x) + \eta_t^2\| g_t\|^2~.
\end{align*}  
Re-arranging the above we get: 
$$g_t(x_t-x)\leq \frac{1}{2\eta_t}\left(\|x_t-x \|^2-\|x_{t+1}-x \|^2\right) +\frac{\eta_t}{2}\| g_t\|^2~.$$
Combining the above with the $H_t$-strong-convexity of $f_t$ and summing over all rounds we conclude that,
\begin{align*}
\sum_{t=1}^T f_t(x_t)-\sum_{t=1}^Tf_t(x)
\le
 \sum_{t=1}^T \frac{\|x_t-x \|^2}{2}\left(\frac{1}{\eta_t}-\frac{1}{\eta_{t-1}}-H_t\right)
+\sum_{t=1}^T\frac{\eta_t}{2}\| g_t\|^2~,
\end{align*}
where we denote $\eta_0 =\infty$. Recalling $\eta_t =(\sum_{\tau=1}^t H_\tau)^{-1}$, the lemma follows.
\end{proof}

%%%Log-Sum Proof
%%%%%%%%%%%%%%%%%%%%%%%%%%%%%%%%%%%%%%%%%%%%%%%%%%%%%%%%%%%%%%
%%%%%%%%%%%%%%%%%%%%%%%%%%%%%%%%%%%%%%%%%%%%%%%%%%%%%%%%%%%%%%
\subsection{Proof of Lemma~\ref{lem:Log_sum}}\label{sec:Proof_lem:Log_sum}
\begin{proof}
We will prove the statement by induction over $n$. 
The base case $n=1$ naturally holds. For the induction step, let us assume that the guarantee holds for $n-1$, which implies that for any $a_1,\ldots, a_n\geq 1$,
\begin{align*}
\sum_{i=1}^{n} \frac{a_i}{\sum_{j=1}^i a_j} 
\le 
1+\log( \sum_{i=1}^{n-1} a_i) + \frac{a_n}{\sum_{i=1}^n a_i}~.
\end{align*}
The above suggests that establishing following inequality concludes the proof,
\begin{align} \label{eq:Induction}
1+\log( \sum_{i=1}^{n-1} a_i) + \frac{a_n}{\sum_{i=1}^n a_i}
\le 
1+\log( \sum_{i=1}^{n} a_i) ~.
\end{align}
Using the notation  $x = a_n/\sum_{i=1}^{n-1}a_i$, Equation~\eqref{eq:Induction} is equivalent to the following,
\begin{align*} 
 \log(x+1) -\frac{x}{1+x} \ge 0~.
\end{align*}
However, it is immediate to validate that the function $M(x) = \log(x+1) -\frac{x}{1+x}$, is non-negative for any $x\geq 0 $, which establishes the lemma.
\end{proof}

%%%Proof Lazy SGD Stochastic
%%%%%%%%%%%%%%%%%%%%%%%%%%%%%%%%%%%%%%%
%%%%%%%%%%%%%%%%%%%%%%%%%%%%%%%%%%%%%%%
\section{Proofs for Section~\ref{sec:Stochastic Lazy SGD} (Lazy SGD)}
%%%Proof AE Procedure
%%%%%%%%%%%%%%%%%%%%%%%%%%%%%%%%%%%%%%%
%%%%%%%%%%%%%%%%%%%%%%%%%%%%%%%%%%%%%%%
\subsection{Proof of Lemma~\ref{lem:AE_guaranteeNewInformal}} \label{app:AE}
We first provide the exact statement rather than the informal one appearing in Lemma~~\ref{lem:AE_guaranteeNewInformal}.
\begin{lemma} \label{lem:AE_guaranteeNew}
Let $T_{\max}\geq 1$.
Suppose an oracle $\G:\K \mapsto \reals^d$ that generates i.i.d. random vectors with an (unknown) expected value $g \in \reals^d$. Assume that $\text{w.p. } 1$ the Euclidean norm of the sampled vectors is  bounded by $G$. 
Then  w.p.$\geq 1-\delta$, invoking  AE (Algorithm~\ref{algorithm:AE_new}), with 
$m_0 = 6G\left(1+\sqrt{\log(\delta^{-1}(1+\log_2 T_{\max}))}\right)$, it is ensured that:
$$ \min \left\{ {m_0^2}/{\|g\|^2},T_{\max} \right\} \le  N \le \min \left\{ {32m_0^2}/{\|g\|^2},T_{\max} \right\}~.\textbf{        (1)  }$$
%within $ t \leq \min \left\{ {16M^2}/{\|g\|^2},T_{\max} \right\}$ oracle queries.
Moreover, w.p.$\geq 1-\delta$, the following holds for the output of the algorithm:
$$ \sqrt{N}\|\tg_N\| \leq 8m_0~.  \textbf{      (2)         }$$ 
and also,
$$E[N(\tg_N- g)]=0~. \textbf{      (3)         }$$
\end{lemma}

We will require the following Hoeffding type inequality  regarding vector valued random variables, by \cite{Kakade13} (see also \cite{juditsky2008large})
\begin{theorem}\label{thm:VecConcentration1}
Suppose that $X_1,X_2,\ldots, X_n\in \reals^d$ are i.i.d. random vectors, and that $\forall i\in[n];\;\|X_i\|\leq M$ almost surely.
Then w.p.$\geq 1-\delta$\begin{align*}
\left \| \frac{1}{n}\sum_{i=1}^nX_i -\E[X_1]  \right\| \le \frac{6M}{\sqrt{n}}\left(1+ \sqrt{ \log \delta^{-1}} \right)~.
\end{align*}
\end{theorem}
%for completeness we prove the above lemma in Section~\ref{sec:VecHoefding}.

We are now ready to prove Lemma~\ref{lem:AE_guaranteeNew}.
\begin{proof}[Proof of Lemma~\ref{lem:AE_guaranteeNew}]
Define $V = \left\{\{ 2^{i}-1 \}_{i=1}^{\log_2 T_{\max}}, T_{\max}\right\}$, and note
that $N$ is a discrete random variable taking one of the $1+\log_2 T_{\max}$ possible values among $V$.
By Theorem~\ref{thm:VecConcentration1} combined with the union bound, it follows 
 that  w.p.$\geq 1-\delta$,  for every $n\in V $  we have $\|\tg_n- g\|\leq \frac{m_0}{\sqrt{n}}$.
This means  the following to hold:
%, for any $n\in V$ such that $\| g\|\leq M/\sqrt{n}$, it follows that
\begin{align}\label{eq:AE1}
 \| \tg_n \|\leq \|g\| + \|\tg_n-g\|\leq \frac{2m_0}{\sqrt{n}}, \quad 
 \forall n\in V \text{ such that } \| g\|\leq m_0/\sqrt{n}
\end{align}
Furthermore,% for any $n$ such that $\|\bg\|\geq 4M/\sqrt{n}$, it follows that
\begin{align}\label{eq:AE2}
 \| \tg_n \|\geq \|g\| - \|\tg_n-g\|\geq \frac{3m_0}{\sqrt{n}}, \quad 
 \forall n\in V \text{ such that } \|g\|\geq 4m_0/\sqrt{n}
\end{align}
The above together with the stopping criteria of Algorithm~\ref{algorithm:AE_new} directly implies the first part of the lemma.

For the second part of the lemma, recall that $N$ is the total number of samples, and let $N_{\text{prev}}$ be the number of samples up to the iteration before stopping. Then necessarily, $N_{\text{prev}} \geq (N-1)/2$.
Since the loop did not stop at the iteration before setting $N$, it follows that
$\sqrt{ N_{\text{prev}}}\|\tg_{N_{\text{prev}}}\| \leq 3m_0$ (i.e. the stopping criteria of the loop at the round prior to setting $N$ fails).  Recalling that  w.p.$\geq 1-\delta$,  for every $n\in V $  we have $\|\tg_n- g\|\leq \frac{m_0}{\sqrt{n}}$, and combining this with the above implies:

\begin{align*}
\sqrt{N} \|\tg_N\| 
&\le 
 \sqrt{N}\blr{ \|\tg_N-g\| +  \|g-\tg_{N_{\text{prev}}}\|   } +  \sqrt{N}\|\tg_{N_{\text{prev}}}\| \\
&\le
\sqrt{N} \left(\frac{m_0}{\sqrt{N}} +\frac{m_0}{\sqrt{N_{\text{prev}}}} \right) 
+ \sqrt{\frac{N}{ N_{\text{prev}}  }} \sqrt{N_{\text{prev}}}\|\tg_{N_{\text{prev}}}\| \\
&\le
m_0 +\sqrt{3}m_0 + \sqrt{3}\cdot 3m_0 \\
&\le 
8m_0
\end{align*}
Where we have used $N\leq 3\frac{N-1}{2}\leq 3N_{\text{prev}}$;
 which holds since  $N_{\text{prev}} \geq (N-1)/2$ and also $N\geq3$.
The latter is ensured since  for any $n\leq 3$ then $\|\tg_n\|\leq G< 3m_0/\sqrt{n}$.

%\newpage
%For the second part of the lemma, assume for simplicity that $N<T_{\max}$\footnote{The proof for the case where $N=T_{\max}$ follows the same lines.}.
%%consider the following two cases: If the Algorithm's loop runs until the $\textbf{while}$ condition fails, then necessarily 
%%$\sqrt{N}\|\tg_N\| \leq 3m_0\leq 8m_0$. Otherwise, then necessarily $\sqrt{N}\|\tg_N\| > 3m_0$, but also
%Since the loop did not stop at the iteration before setting $N$, it follows that
%$\sqrt{ \frac{N-1}{2}}\|\tg_{(N-1)/2}\| \leq 3m_0$ (i.e. the stopping criteria of the loop at the round prior to setting $N$ fails).  Recalling that  w.p.$\geq 1-\delta$,  for every $n\in V $  we have $\|\tg_n- g\|\leq \frac{m_0}{\sqrt{n}}$, and combining this with the above implies:
%%Combined this with the conditions of Equations~\eqref{eq:AE1}, \eqref{eq:AE2}, implies that w.p.$\geq1-\delta$:
%\begin{align*}
%\sqrt{N} \|\tg_N\| 
%&\le 
% \sqrt{N}\|\tg_N-\tg_{(N-1)/2}\| +  \sqrt{N}\|\tg_{(N-1)/2}\| \\
%&\le
%\sqrt{N} \left(\frac{m_0}{\sqrt{N}} +\frac{m_0}{\sqrt{(N-1)/2}} \right) 
%+ \sqrt{3} \sqrt{\frac{N-1}{2}}\|\tg_{(N-1)/2}\| \\
%&\le
%m_0 +\sqrt{3}m_0 + \sqrt{3}\cdot 3m_0 \\
%&\le 
%8m_0
%\end{align*}
%where we have uses $N\leq 3\frac{N-1}{2}$ which holds since $N\geq3$.
%This is ensured since  for any $n\leq 3$ then $\|\tg_n\|\leq G< 3m_0/\sqrt{n}$.

For the third part of the lemma, it is easy to notice that for any fixed $n$ then 
$n(\tg_n-g)$ is a sum of $n$ i.i.d. random variables, and that $\E[n(\tg_n-g)]=0$. Since $N$ is a bounded stopping time, Doob's optional stopping theorem~\cite{levin2009markov} implies that $E[N(\tg_N- g)]=0$.
\end{proof}

%%%Proof Lazy SGD General
%%%%%%%%%%%%%%%%%%%%%%%%%%%%%%%%%%%%%%%
%%%%%%%%%%%%%%%%%%%%%%%%%%%%%%%%%%%%%%%
\subsection{Proof of Lemma~\ref{lem:LazySGD_general_expectation}}\label{app:lazysgd1}
 \label{Proof_lem:LazySGD_general_expectation}
\begin{proof}
Let $S$ be the total number of times that LazySGD invokes the AE procedure.
We will first upper bound the expectation of following sum (weighted regret):
\begin{align} \nonumber 
\sum_{s=1}^S n_s\left( f(x_s) - f(x^*)\right) 
&\le
\sum_{s=1}^S n_s g_s^\top(x_s - x^*) \nonumber\\
&\le
 \underbrace{\sum_{s=1}^S n_s \tg_s^\top(x_s - x^*)}_{(a)} +  \underbrace{\sum_{s=1}^S n_s (g_s-\tg_s)^\top(x_s - x^*)}_{(b)} \label{eq:LazySGDcvx1_Ex}
\end{align}
where we have used the gradient inequality.
The proof goes on by bounding the expectation of terms $(a)$, $(b)$ appearing above.
\paragraph{Bounding term (a):}
Assume that LazySGD uses the AE procedure with some $\delta>0$.
Since LazySGD is equivalent to $\text{AdaNGD}_2$ with $\|g_s\|^2 \gets n_s$ and $g_s \gets n_sg_s$, then a similar analysis to $\text{AdaNGD}_2$ may show that this sum is bounded by $O(\sqrt{T})$. For completeness we provide the full analysis here.
Consider the update rule of LazySGD: $x_{s+1}=\Pi_\K(x_s - \eta_s n_s\tg_s)$. We can write:
\begin{align*}
\|x_{s+1}-x^* \|^2 \leq \|x_s-x^* \|^2 -2\eta_s n_s\tg_s^\top(x_t-x^*) + \eta_s^2 n_s^2\| \tg_s\|^2
\end{align*}  
Re-arranging the above we get: 
$$n_s \tg_s^\top(x_s-x^*)\leq \frac{1}{2\eta_s}(\|x_s-x^* \|^2-\|x_{s+1}-x^* \|^2) +\frac{\eta_s}{2}n_s^2\| \tg_s\|^2$$
Summing over all rounds we conclude that w.p.$\geq1-\delta T$:
\begin{align*}
\textbf{(a)} 
&\eq
\sum_{s=1}^Sn_s \tg_s^\top(x_s - x^*)  \\
&\le
 \sum_{s=1}^S \frac{\|x_s-x^* \|^2}{2}(\frac{1}{\eta_s}-\frac{1}{\eta_{s-1}})
+\sum_{s=1}^S\frac{\eta_s}{2}n_s^2\| \tg_s\|^2  \\
&\le
 \frac{D^2}{2}\sum_{s=1}^S(\frac{1}{\eta_s}-\frac{1}{\eta_{s-1}})
+64m_0^2 \sum_{s=1}^S \eta_s n_s  \\
&\le 
\frac{DG}{2}\sqrt{2T}  
+\frac{64m_0^2D}{G} \sum_{s=1}^S \frac{n_s}{\sqrt{\sum_{i=1}^s n_i}}  \\
&\eq 
\frac{DG}{2}\sqrt{2T} 
+\frac{128 m_0^2D}{G}   \sqrt{\sum_{s=1}^S n_s}   \\
&\le
O(GD\sqrt{T}\log(1/\delta))~.
\end{align*}
here in the first inequality we denote $\eta_0 =\infty$, 
the second inequality uses  $n_s\|\tg_s\|^2\leq 64m_0^2$, which follows by Theorem~\ref{lem:AE_guaranteeNew}, and it also uses $\eta_s\leq \eta_{s-1}$; the fourth inequality uses  Lemma~\ref{lem:SqrtSum}. We also make use of $\sum_{s=1}^Sn_s=T$, and 
$1/\eta_s = \sqrt{\sum_{i=1}^s n_i}$.

Since $\textbf{(a)}$ is bounded by $2GDT$, then taking $\delta =1/T^{3/2}$ ensures that,
\begin{align}\label{eq:LazySGDcvx2_Ex}
\E[\textbf{(a)}] \le  O(GD\sqrt{T}\log(T))~.
\end{align}
\paragraph{Bounding term (b):}
Here we show  that $\E[\textbf{(b)}]=0$.
Without loss of generality we will make the following two assumptions which do not affect the output of LazySGD:
\begin{itemize}
\item We assume that LazySGD invokes the AE procedure exactly $T$ times.
Note that in practice the algorithm invokes the AE procedure $S$ times, where $S\leq T$ is a random variable, after which $T-t=0$. Nevertheless calling AE for any $s\in\{S+1,\ldots,T\}$ yields $\tg_s=0, n_s=0$, which  does not affect the output of LazySGD. 
\item We assume that at each time $s\in[T]$ that LazySGD calls the AE procedure, it samples exactly $T$ times from $\text{GradOracle}(x_s)$.
We denote these samples by $\{\tg_s^{(i)}\}_{i=1}^T$.
Nevertheless the output of the procedure only uses the first $n_s$ samples, where $n_s$ is set according to the AE procedure. Thus the remaining $T-n_s$ samples do not affect the output of AE and LazySGD. Note that $\forall s\in[T],\; n_s\leq T-t\leq T$,
\end{itemize}
Thus, for any $s\in[T]$ let $\{\tg^{(i)}_s\}_{i=1}^{T}$ be the samples drawn from the noisy first order oracle $\text{GradOracle}(x_s)$ during the $s$'th call to AE at this iteration. This implies that $n_s\tg_s =  \sum_{i=1}^{n_s} \tg^{(i)}_s$. Term $(b)$ can be therefore written as follows:
\begin{align*}
\textbf{(b)}
\eq
\sum_{s=1}^T n_s (g_s-\tg_s)^\top(x_s - x^*) 
\eq
\sum_{s=1}^T \sum_{i=1}^{n_s} (g_s-\tg_s^{(i)})^\top(x_s - x^*) 
\end{align*}
%Since $\E[\G(x_s) \vert x_s] = g_s$, then for  fixed  $s$ and $i$, every term in the above sum is  unbiased,
%$$ \E[(g_s-\tg_s^{(i)})^\top(x_s - x^*)] \eq \E[ \E[(g_s-\tg_s^{(i)})^\top(x_s - x^*) \vert x_s] ] \eq 0~.$$
Given $s\in[T]$ define the following filtration:
\begin{align*}
\F_0^{(s)} &\eq \sigma\text{-field}\left\{x_s,t\right\} \\
\F_j^{(s)} & \eq  \sigma\text{-field}\left\{x_s,t,g_s^{(1)},\ldots,g_s^{(j)}\right\},\quad \forall j\in[T]
\end{align*}
Also define the following sequence  $\{B_j^{(s)}\}_{j=0}^T$:
\begin{align*}
B_0^{(s)} \eq 0, \qquad
B_j^{(s)} \eq \sum_{i=1}^j (g_s-\tg_s^{(i)})^\top(x_s - x^*), \quad \forall j\in[T]
\end{align*}
Since $\E[\tg_s^{(i)} \vert x_s] = g_s,\; \forall i, s\in[T]$, then it immediately follows that $\{B_j^{(s)}\}_{t=0}^T$ is a martingale with respect to the above filtration. Also it is immediate to see that $n_s$ is a bounded stopping time
with respect to the above filtration.
Thus, Doob's optional stopping theorem (see~\cite{levin2009markov}) implies that 
$$\E[B_{n_s}^{(s)}\vert \F_0] \eq \E\left[\sum_{i=1}^{n_s} (g_s-\tg_s^{(i)})^\top(x_s - x^*) \vert \F_0\right] \eq 0~. $$
which directly implies,
\begin{align*} 
\E[\textbf{(b)}] 
\eq
\E\left[ \sum_{s=1}^T B_{n_s}^{(s)}\right]
\eq 0 ~.
\end{align*}
Using Jensen's inequality and 
combining the above with Equations~\eqref{eq:LazySGDcvx1_Ex}, \eqref{eq:LazySGDcvx2_Ex},  
establishes the lemma:
\begin{align*} 
\E[f(\bar{x}_T)] - f(x^*) 
&\le
 \E\left[ \sum_{s=1}^S\frac{n_s}{T}\left( f(x_s) - f(x^*)\right)\right] \\
&\le
\frac{1}{T} O(GD\sqrt{T}\log(T))\\
&\le
 O({GD\log(T)}/{\sqrt{T}})~.
\end{align*}
\end{proof}

%%%Proof Lazy Strongly Convex
%%%%%%%%%%%%%%%%%%%%%%%%%%%%%%%%%%%%%%%
%%%%%%%%%%%%%%%%%%%%%%%%%%%%%%%%%%%%%%%
\subsection{Proof of Lemma~\ref{lem:LazySGDstronglyConvex_expectation}}\label{app:lazysgd2}
 \label{Proof_lem:LazySGDstronglyConvex_expectation}
\begin{proof}
Let $S$ be the total number of times that LazySGD invokes the AE procedure.
We will first upper bound the expectation of the following sum (weighted regret):
\begin{align} \nonumber \label{eq:LazySGDStronglycvx1_Ex}
\sum_{s=1}^S& n_s\left( f(x_s) - f(x^*)\right) \\
&\le
\sum_{s=1}^S n_s( g_s^\top(x_s - x^*) -\frac{H}{2}\|x_s-x^*\|^2)\nonumber\\
&\le
 \underbrace{\sum_{s=1}^S n_s( \tg_s^\top(x_s - x^*) -\frac{H}{2}\|x_s-x^*\|^2)}_{(a)} +  \underbrace{\sum_{s=1}^S n_s (g_s-\tg_s)^\top(x_s - x^*)}_{(b)} 
\end{align}
where we have used the $H$-strong-convexity of $f(\cdot)$.
The proof goes on by bounding the expectation of terms $(a)$, $(b)$ appearing above.
\paragraph{Bounding term (a):}
Assume that LazySGD uses the AE procedure with some $\delta>0$.
Since LazySGD is equivalent to $\text{SC-AdaNGD}_2$ with $\|g_s\|^2 \gets n_s$ and $g_s \gets n_sg_s$, then a similar analysis to $SC-\text{AdaNGD}_2$ may show that this sum is bounded by $O(\log T)$. For completeness we provide the full analysis here.
Consider the update rule of LazySGD: $x_{s+1}=\Pi_\K(x_s - \eta_s n_s\tg_s)$. We can write:
\begin{align*}
\|x_{s+1}-x^* \|^2 \leq \|x_s-x^* \|^2 -2\eta_s n_s\tg_s^\top(x_t-x^*) + \eta_s^2 n_s^2\| \tg_s\|^2
\end{align*}  
Re-arranging the above we get: 
$$n_s \tg_s^\top(x_s-x^*)\leq \frac{1}{2\eta_s}(\|x_s-x^* \|^2-\|x_{s+1}-x^* \|^2) +\frac{\eta_s}{2}n_s^2\| \tg_s\|^2$$
Summing over all rounds we conclude that w.p.$\geq1-\delta T$:
\begin{align}\label{eq:LazySGD_Stronglycvx2_Ex}
\textbf{(a)} 
&\eq
\sum_{s=1}^Sn_s \tg_s^\top(x_s - x^*) -n_s\frac{H}{2}\|x_s-x^*\|^2 \nonumber\\
&\le
 \sum_{s=1}^S \frac{\|x_s-x^* \|^2}{2}(\frac{1}{\eta_s}-\frac{1}{\eta_{s-1}}-n_sH)
+\sum_{s=1}^S\frac{\eta_s}{2}n_s^2\| \tg_s\|^2  \nonumber\\
&\leq
0 + 32m_0^2 \sum_{s=1}^S\eta_sn_s  \nonumber\\
&\leq 
 \frac{32m_0^2}{H} \sum_{s=1}^S\frac{n_s}{\sum_{k=1}^s n_s}  \nonumber\\
 &\eq 
\frac{32m_0^2}{H} (1+\log(\sum_{s=1}^Sn_s))  \nonumber \\
&\le
\tO(\frac{G^2\log T}{H} \log(1/\delta))~.
\end{align}
here in the first inequality we denote $\eta_0 =\infty$, 
the second inequality uses $1/\eta_s = H\sum_{i=1}^s n_i$,  and also $n_s\|\tg_s\|^2\leq 64m_0^2$, which follows by Theorem~\ref{lem:AE_guaranteeNew}; the fourth inequality uses  Lemma~\ref{lem:Log_sum}. We also make use of $\sum_{s=1}^Sn_s=T$.\\
Since $\textbf{(a)}$ is bounded by $2GDT$, then taking $\delta =O(1/T^{2})$ ensures that,
\begin{align}\label{eq:LazySGDStrcvx2_Ex}
\E[\textbf{(a)}] \le  O(G^2\log^2(T)/H)~.
\end{align}

\paragraph{Bounding term (b):}
Similarly the proof of Lemma~\ref{lem:LazySGD_general_expectation} (see  Section~\ref{Proof_lem:LazySGD_general_expectation}) we can show that,
$$\E[\textbf{(b)}] \eq 0~. $$
Using Jensen's inequality and 
combining the above with Equations~\eqref{eq:LazySGDStronglycvx1_Ex}
,\eqref{eq:LazySGDStrcvx2_Ex},
establishes the lemma:
\begin{align*} 
\E[f(\bar{x}_T)] - f(x^*) 
&\le
 \E[\sum_{s=1}^S\frac{n_s}{T}\left( f(x_s) - f(x^*)\right)] \\
&\le
O\left(\frac{G^2\log^2(T)}{HT}\right)~.
\end{align*}

\end{proof}

\end{document}